\documentclass{article}

\usepackage{microtype}
\usepackage{graphicx}
\usepackage{subfigure}
\usepackage{booktabs} %
\usepackage{bbm}
\usepackage{hyperref}

\usepackage[accepted]{icml2024}

\usepackage{amsmath}
\usepackage{amssymb}
\usepackage{mathtools}
\usepackage{amsthm}
\usepackage{multirow}

\usepackage[capitalize,noabbrev]{cleveref}

\usepackage{tikz}
\usetikzlibrary{shapes}
\usetikzlibrary{positioning}
\usetikzlibrary{arrows.meta}
\tikzset{global scale/.style={
    scale=#1,
    every node/.append style={scale=#1}
  }
}

\theoremstyle{plain}
\newtheorem{theorem}{Theorem}[section]

\newtheorem{lemma}[theorem]{Lemma}

\theoremstyle{definition}

\theoremstyle{remark}

\usepackage{amsmath,amsfonts,bm}

\def\eqref#1{equation~\ref{#1}}

\def\1{\bm{1}}

\DeclareMathAlphabet{\mathsfit}{\encodingdefault}{\sfdefault}{m}{sl}
\SetMathAlphabet{\mathsfit}{bold}{\encodingdefault}{\sfdefault}{bx}{n}

\def\gD{{\mathcal{D}}}

\def\gG{{\mathcal{G}}}

\def\gL{{\mathcal{L}}}

\def\gP{{\mathcal{P}}}

\newcommand{\E}{\mathbb{E}}

\newcommand{\R}{\mathbb{R}}

\renewcommand{\epsilon}{\varepsilon}

\newcommand{\bA}{\mathbf{A}}

\newcommand{\bG}{\mathbf{G}}
\newcommand{\bH}{\mathbf{H}}

\newcommand{\bK}{\mathbf{K}}

\newcommand{\bM}{\mathbf{M}}

\newcommand{\bQ}{\mathbf{Q}}

\newcommand{\bV}{\mathbf{V}}
\newcommand{\bW}{\mathbf{W}}

\usepackage[textsize=tiny]{todonotes}

\icmltitlerunning{A Theoretical Comparison Between Autoregressive and Masked Pretraining}

\begin{document}

\twocolumn[
\icmltitle{Look Ahead or Look Around?\\
A Theoretical Comparison Between Autoregressive and Masked Pretraining
           }

\icmlsetsymbol{equal}{*}

\begin{icmlauthorlist}
\icmlauthor{Qi Zhang}{equal,yyy}
\icmlauthor{Tianqi Du}{equal,yyy}
\icmlauthor{Haotian Huang}{comp}
\icmlauthor{Yifei Wang}{sch}
\icmlauthor{Yisen Wang}{yyy,ai}
\end{icmlauthorlist}

\icmlaffiliation{yyy}{National Key Lab of General Artificial Intelligence, School of Intelligence Science and Technology, Peking University, China}
\icmlaffiliation{comp}{Sun Yat-Sen University,  China}
\icmlaffiliation{sch}{MIT CSAIL, USA}
\icmlaffiliation{ai}{Institute for Artificial Intelligence, Peking University, China}

\icmlcorrespondingauthor{Yisen Wang}{yisen.wang@pku.edu.cn}

\icmlkeywords{Machine Learning, ICML}

\vskip 0.3in
]

\printAffiliationsAndNotice{\icmlEqualContribution} %

\begin{abstract}
    In recent years, the rise of generative self-supervised learning (SSL) paradigms has exhibited impressive performance across visual, language, and multi-modal domains. While the varied designs of generative SSL objectives lead to distinct properties in downstream tasks, a theoretical understanding of these differences remains largely unexplored. In this paper, we establish the first theoretical comparisons between two leading generative SSL paradigms: autoregressive SSL and masked SSL. Through establishing theoretical frameworks, we elucidate the strengths and limitations of autoregressive and masked SSL within the primary evaluation tasks of classification and content generation. Our findings demonstrate that in classification tasks, the flexibility of targeted tokens in masked SSL fosters more inter-sample connections compared to the fixed position of target tokens in autoregressive SSL, which yields superior clustering performance. In content generation tasks, the misalignment between the flexible lengths of test samples and the fixed length of unmasked texts in masked SSL (vs. flexible lengths of conditional texts in autoregressive SSL) hinders its generation performance. To leverage each other's strengths and mitigate weaknesses, we propose diversity-enhanced autoregressive and variable-length masked objectives, which substantially improve the classification performance of autoregressive SSL and the generation performance of masked SSL. Code is available at \url{https://github.com/PKU-ML/LookAheadLookAround}.

\end{abstract}

\section{Introduction}

Recently, self-supervised learning (SSL) paradigms have emerged to be promising in a wide variety of domains \citep{mae, gpt,bert,clip}. 
Generative SSL, characterized by its ability to reconstruct target tokens from conditional tokens within the same sequence, stands out for its distinct advantages in fine-tuning \citep{mae}, zero-shot learning \citep{gpt3,clip}, and many other downstream tasks. 
In general, generative SSL can be divided into two categories: autoregressive SSL (represented by GPT \citep{gpt}) and masked SSL (represented by BERT \citep{bert}). As shown in Figure \ref{fig:two_paradigms}, the key difference between them lies in the choice of target tokens. Taking the language models as an example, autoregressive SSL predicts the next word in a sentence based on preceding words, while masked SSL predicts randomly masked words in a sentence based on bidirectional words. As a result, these two generative SSL models demonstrate significant differences in many aspects. For example, on two primary evaluation tasks of pretrained models, i.e., classification and content generation tasks, masked SSL tends to perform better in classification tasks, whereas autoregressive SSL shows superior performance in content generation tasks (Table \ref{tab:cls_and_gen}). Despite empirical evidence supporting their differing efficacies, the underlying reasons explaining the distinct properties of autoregressive and masked SSL remain unclear.

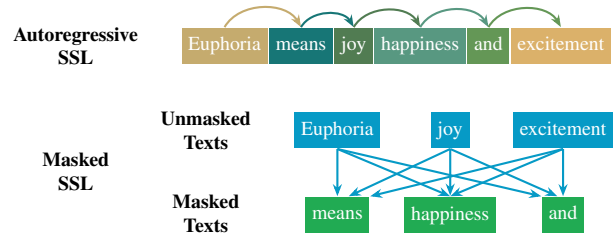
\begin{figure}
    \centering
    \begin{tikzpicture}[global scale=0.75]
    \definecolor{c1}{RGB}{197,171,110}
    \definecolor{c2}{RGB}{23,118,118}
    \definecolor{c3}{RGB}{81,124,82}
    \definecolor{c4}{RGB}{88,152,129}
    \definecolor{c5}{RGB}{100,151,86}
    \definecolor{c6}{RGB}{219,177,106}
    \definecolor{c11}{RGB}{6,154,198}
    \definecolor{c21}{RGB}{33,171,83}
    \node[thick,fill=c1,text=white, minimum height=4ex](t1) at (0,0){Euphoria};
    \node[thick,fill=c2,text=white, right=0 and 0 of t1,minimum height=4ex](t2) {means};
    \node[thick,fill=c3,text=white, right=0 and 0 of t2,minimum height=4ex](t3) {joy};
    \node[thick,fill=c4,text=white, right=0 and 0 of t3,minimum height=4ex](t4) {happiness};
    \node[thick,fill=c5,text=white, right=0 and 0 of t4,minimum height=4ex,text height=1ex](t5) {and};
    \node[thick,fill=c6,text=white, right=0 and 0 of t5,minimum height=4ex,text height=1ex](t6) {excitement};
    \node[left=0 and 3ex of t1, align=center,font=\bf](ar) {Autoregressive\\SSL};
    \node[below=2.8em and 0 of ar, align=center,font=\bf](mask) {Masked\\SSL};
    \node[thick,fill=c11,text=white, minimum height=4ex](t11) at (2,-1.5) {Euphoria};
    \node[thick,fill=c21,text=white, minimum height=4ex](t21) at (2,-3){means};
    \node[thick,fill=c11,text=white, minimum height=4ex](t31) at (4,-1.5){joy};
    \node[thick,fill=c21,text=white, minimum height=4ex](t41) at (4,-3) {happiness};
    \node[thick,fill=c21,text=white, minimum height=4ex,text height=1ex](t51) at (6,-3) {and};
    \node[thick,fill=c11,text=white, minimum height=4ex,text height=1ex](t61) at (6,-1.5){excitement};
    \node[left=0 and 3ex of t11, align=center,font=\bf](ut) {Unmasked\\Texts};
    \node[below=1.2em and 0 of ut, align=center,font=\bf]{Masked\\Texts};
    \draw[arrows = {-Stealth[scale=0.7]},thick,draw=c1] (t1.north) to[out=60,in=120] (t2.north);
    \draw[arrows = {-Stealth[scale=0.7]},thick,draw=c2] (t2.north) to[out=60,in=120] (t3.north);
    \draw[arrows = {-Stealth[scale=0.7]},thick,draw=c3] (t3.north) to[out=60,in=120] (t4.north);
    \draw[arrows = {-Stealth[scale=0.7]},thick,draw=c4] (t4.north) to[out=60,in=120] (t5.north);
    \draw[arrows = {-Stealth[scale=0.7]},thick,draw=c5] (t5.north) to[out=60,in=120] (t6.north);
    \draw[arrows = {-Stealth[scale=0.7]},thick,draw=c11] (t11.south) to (t21.north);
    \draw[arrows = {-Stealth[scale=0.7]},thick,draw=c11] (t11.south) to (t41.north);
    \draw[arrows = {-Stealth[scale=0.7]},thick,draw=c11] (t11.south) to (t51.north west);
    \draw[arrows = {-Stealth[scale=0.7]},thick,draw=c11] (t31.south) to ([xshift=2mm]t21.north);
    \draw[arrows = {-Stealth[scale=0.7]},thick,draw=c11] (t31.south) to ([yshift=0.1mm]t41.north);
    \draw[arrows = {-Stealth[scale=0.7]},thick,draw=c11] (t31.south) to ([xshift=-1mm]t51.north);
    \draw[arrows = {-Stealth[scale=0.7]},thick,draw=c11] (t61.south) to (t21.north east);
    \draw[arrows = {-Stealth[scale=0.7]},thick,draw=c11] (t61.south) to (t41.north);
    \draw[arrows = {-Stealth[scale=0.7]},thick,draw=c11] (t61.south) to (t51.north);
    \end{tikzpicture}
    \caption{Illustration of two primary paradigms in generative SSL: autoregressive SSL and masked SSL.}
    \label{fig:two_paradigms}
    \vskip -0.2in
\end{figure}

In this paper, we propose the first theoretical comparison between autoregressive and masked SSL, with a  focus on classification and content generation tasks as widely considered in the literature \citep{gpt,bert,dong2019unified}. To analyze the classification performance of autoregressive and masked SSL uniformly, we establish the unified connections between generative SSL objectives and matrix decomposition objectives. Based on these connections, we establish a theoretical downstream guarantee for generative SSL and describe the downstream classification performance by the properties of generative SSL co-occurrence matrices. From this perspective, we conduct both empirical and theoretical comparisons of autoregressive and masked SSL. Our analysis demonstrates that randomly selected target tokens in masked SSL could lead to more connections between different semantically similar samples, leading to superior downstream classification performance compared with autoregressive SSL. To bridge this gap, for autoregressive SSL we propose a diversity-enhanced autoregressive objective, which improves its classification performance in real-world datasets.

\begin{table}
        \centering
        \caption{The autoregressive and masked SSL models with similar scales of parameters show distinct properties in classification (evaluated by GLUE score \citep{wang2018glue}) and content generation tasks (evaluated by perplexity on WikiText-2 \citep{merity2016wikitext}). More details can be found in Appendix \ref{appendix-comparison-detail}.}
        \vskip 0.1in
        \begin{tabular}{lcc}
        \toprule
        Pretraining & GLUE ($\uparrow$) & Perplexity ($\downarrow$) \\
            \midrule
            Autoregressive & 72.8 & \textbf{22.8}\\
            Masked  & \textbf{80.5} & 34.1\\
        \bottomrule
        \end{tabular}
        \label{tab:cls_and_gen}
\end{table}

On content generation tasks \citep{dong2019unified}, we first compare the generation ability of masked and autoregressive SSL on different sample lengths.  
We observe that the masked SSL exhibits particularly inferior performance when generating short texts. Intuitively, it is attributed to a misalignment between the pretraining and downstream objectives. For instance, in a question-answering scenario, the length of questions is variable while the length of masked texts in the masked SSL objective is fixed. As a result, the model cannot find the specific part of the information that is used to infer the masked texts during the pretraining process and may output a random guess. To formalize this intuition, we theoretically compare autoregressive and masked SSL generation abilities in the sample with different lengths. We prove that the performance gap is decided by the degree of the misalignment between pretraining objectives and downstream evaluation tasks. Building on this analysis, for masked SSL we introduce variable-length masked objectives, which exhibit significant improvements in its content generation performance.

We summarize our contributions as follows:
\begin{itemize}
    \item We establish the first unified theoretical framework to compare the downstream classification performance of autoregressive and masked SSL, showcasing the unique advantage of masked SSL in enhancing sample connectivity.
    
    \item We build a theoretical comparison in generation abilities between autoregressive and masked SSL, which reveals that masked SSL's underperformance is attributed to the misalignment between the pretraining and downstream objectives. 
    \item To further verify the theoretical insights, we propose the diversity-enhanced autoregressive and variable-length masked objectives, which substantially improve the classification performance of autoregressive SSL and the generation performance of masked SSL.
\end{itemize}

\section{Related Work}

\textbf{Self-supervised Learning.} Traditional deep learning relies on label information to cluster semantically similar samples in the feature space \cite{alexnet}. However, due to the expensive costs of labeled data, the performance of deep learning is constrained. As a result, self-supervised (SSL)  paradigms are proposed to extract meaningful representations from unlabeled data. Among various self-supervised objectives (e.g., generative SSL 
objectives \citep{bert,mae}, contrastive SSL objectives \citep{simclr,giorgi2020declutr} and context-based SSL objectives \citep{larsson2016learning,gidaris2018unsupervised}), generative paradigms have achieved impressive performance in vision, language, and multi-modal domains.

\textbf{Generative SSL.} Generative SSL objectives can be broadly categorized into two main types: autoregressive SSL (predicting the next token of the prefix information) and masked SSL (predicting random tokens with the bidirectional information). In language models, the autoregressive representative model GPT \citep{gpt} and its variants \citep{radford2019language,gpt3} obtain promising performance in content generation, long-text comprehension, and many other downstream tasks. Concurrently, the masked SSL representative language model BERT \citep{bert} and its variants \citep{lan2019albert,liu2019roberta} also show unique advantages and achieve remarkable performance in downstream tasks including sequence labeling, text classification, etc. Besides the language tasks, the generative SSL objectives have also shown great potential in visual representation learning \citep{mae,xie2022simmim} and multi-modal representation learning \citep{bachmann2022multimae}. 

\textbf{Theoretical Understanding of SSL.} Despite the impressive empirical performance exhibited by various self-supervised objectives in diverse downstream tasks, the theoretical understanding of the mechanisms behind them (especially for the generative SSL objectives) is still under-explored. In the language domain, previous studies mainly demonstrated the advantages of different pretraining objectives through empirical comparisons. For example, \citet{wang2022language} shows that the autoregressive SSL models can achieve superior performance in zero-shot tasks, while the masked SSL models perform better in fine-tuning tasks. Besides, \citet{dong2019unified} observes that the autoregressive SSL is more proficient in generation tasks, while the masked SSL outperforms in classification tasks. In the vision domain, researchers mainly focus on the theoretical analysis of contrastive objectives instead of generative SSL objectives. For example, they have established the theoretical connections between the contrastive objectives and the downstream performance \citep{arora,haochen} or analyzed them from a mutual information perspective \citep{oord2018representation, InfoMin}. Different from prior studies, our work mainly focuses on uncovering the mechanism behind the generative SSL objectives by theoretically comparing the two major generative SSL objectives.

\section{Mathematical Formulation}

We begin by introducing the basic notations of autoregressive and masked SSL. Without loss of generality, we take the language models as an example, and the analysis can be easily extended to visual and other domains. In general, both autoregressive and masked SSL comprise two stages: self-supervised pretraining and downstream evaluation. 

\subsection{Pretraining Objectives} 
Taking an unsupervised language dataset $\gD = \{x_i\}_{i=1}^N$ as an example, each sample $x_i=(x_{i,1},\cdots, x_{i,s})$ consists of $s$ tokens. The objective of autoregressive SSL is to predict the subsequent word based on prefix texts, i.e., minimizing the following Negative Log-Likelihood (NLL) Loss:
\begin{align*}
    \gL_{ar}(\Theta) = -\E_{x_i}\sum_{k}\log P(x_{i,k}|x_{i,1},\cdots,x_{i,k-1};\Theta),
\label{eqn:ar loss}
\end{align*}
where  $\Theta$ are the weights of the neural network.

During the masked SSL pretraining process, for each sample $x_i = (x_{i,1},\cdots, x_{i,s})$, we draw a random mask $m \in \{0,1\}^s$ (drawing 0 with probability $\rho_m$, i.e., the mask ratio) and apply the mask to generate the unmasked text $x^1_{i}$ and masked text $x^2_{i}$:
\begin{align*}
    x^1_{i} = x_i[m]\in \R^{s(1-\rho_m) }, x^2_{i} = x_{i}[1-m]\in \R^{s\rho_m }.
\end{align*}
The objective of masked SSL is to predict the masked tokens from unmasked texts, i.e.,
\begin{small}
\begin{align*}
    \gL_{m}(\Theta) = -\E_{(x_{i}^1,x_{i}^2)}\sum_{k}\log P(x^2_{i,k}|x^1_{i,1},\cdots,x^1_{i,s(1-\rho_m)};\Theta),
\end{align*}
\end{small}
\vskip -0.2in

where $x^1_{i,k},x^2_{i,k}$ respectively denote the $k$-th token of the unmasked text $x^1_{i}$ and masked text $x^2_{i}$. 

In practice, both the representative autoregressive SSL model GPT \citep{gpt} and the representative masked SSL model BERT \citep{bert} use the multi-layer Transformer to generate the predicted distribution and apply the Cross-Entropy (CE) loss. Let $f(X_i)\in \R^{t}$ be the output of the network and $W\in \R^{t} \times \R^{N_D}$ be the token embedding matrix that transforms the output features to the predicted distribution of $N_D$ different tokens. Then the predicted probability in autoregressive and masked SSL can be uniformly reformulated as:
\begin{align*}
 P(X_i^+|X_i;\Theta) = \frac{\exp((Wf(X_i))^\top\mathbbm{1}_{X_i^+})}{\sum_{X_i^-} {\exp((W f(X_i))^\top\mathbbm{1}_{X_i^-}}) },   
\end{align*}
where $X_i$ are the conditional tokens, $X_i^+$ are the target tokens and $X_i^-$ are the independent tokens. 
To be specific, taking the sentence in Figure \ref{fig:two_paradigms} as an example ("Euphoria means joy, happiness and excitement"), for autoregressive SSL, $X_i$ represents the prefix tokens (e.g., [Euphoria, means, joy, happiness, and]) and $X_i^+$ is the subsequent token ([excitement]). For masked SSL, $X_i$ represents the unmasked tokens (e.g., [Euphoria,  joy, excitement]) while $X_i^+$ represents one of the masked tokens (e.g., [happiness]). Besides, we use a one-hot vector $\mathbbm{1}_{X_i^+}$ to indicate the target tokens (the $X_{i}^+$-th position of $\mathbbm{1}_{X_i^+}$ is set to 1 ).
The corresponding objective is:
\vskip -20 pt
\begin{footnotesize}
\begin{align*}
    \gL_{CE}(f,W) =  -\E_{(X_i,X_i^+)}\log \frac{\exp((Wf(X_i))^\top\mathbbm{1}_{X_i^+})}{\sum_{X_i^-} {\exp((W f(X_i))^\top\mathbbm{1}_{X_i^-}}) },
\end{align*}
\end{footnotesize}
\vskip -0.1in

We note that the main objective of generative SSL loss is to maximize the probability of target words and minimize the independent word probability. Following the same spirits, 
we consider the following spectral loss \citep{haochen} to simplify our theoretical analysis:
\begin{equation}
\begin{aligned}
    \gL_{SL} (f,W)
    = &-\E_{(X_i,X_i^+)}  (Wf(X_i))^\top\mathbbm{1}_{X_i^+}\\ &+\E_{X_i,X_i^-} ((Wf(X_i))^\top \mathbbm{1}_{X_i^-})^2.
\end{aligned}
\end{equation}

\subsection{Revisiting Objectives from a Matrix Perspective}
When theoretically comparing the generalization performance of autoregressive and masked SSL pretraining, it's challenging to analyze the overall optimal solutions for their objectives as both of them are usually seen as instance-level tasks, which hinders us from demonstrating the downstream performance of pretrained models. Inspired by the previous work \citep{haochen} which addressed analogous challenges in contrastive learning, we consider reformulating the spectral loss and analyzing it from a matrix decomposition perspective.

We start by introducing the co-occurrence matrix $A$, which formulates the joint distribution $P_M(X_i,X_i^+)$ of the conditional texts and target texts as a matrix, i.e.,
\begin{align*}
    A_{X_i,X_i^+} = P_M(X_i,X_i^+) \geq 0. 
\end{align*}
And we denote $P_C(X_i)$$=\sum_{X_i^+}P_M(X_i,X_i^+)$, $P_G(X_i^+)=\sum_{X_i}P_M(X_i,X_i^+)$ as the marginal distributions of conditional and target texts. Based on these definitions, we establish a crucial connection: the spectral loss is equivalent to an asymmetric matrix decomposition objective on the co-occurrence matrix.
\begin{theorem}
Let $\bar{A}$ be the normalized co-occurrence matrix, i.e., $\bar{A}_{X_i,X_i^+} = \frac{A_{X_i,X_i^+}}{\sqrt{P_C(X_i)P_G(X_i^+)}}$. Then we obtain
\begin{align*}
    \gL_{SL}(f,W) = \Vert \bar{A} - FW'^\top \Vert^2 + const.
\end{align*}
where the $X_i$-th row of $F$ and the $X_i^+$-th row of $W'$ respectively represents encoded features and token embeddings, i.e., $F_{X_i} = \sqrt{P_C(X_i)}f(X_i)^\top$, $W'_{X_i^+} = \sqrt{P_G(X_i^+)}W_{X_i^+}$. 
\label{thm:equivalence}
\end{theorem}

\subsection{Downstream Tasks}
Classification tasks and content generation tasks are the most commonly adopted evaluation protocols for pretrained models \citep{gpt,gpt3,dong2019unified,bert,yang2020xlnet,imagegpt,mae}. Therefore, in this paper, we mainly consider the classification tasks and content generation tasks for generality.

\textbf{Linear classification.} For the classification tasks, we assume that the labels of $\gD$ can be accessed, i.e., $\gD = \{ (x_i, y(x_i))\}_{i=1}^N$, where $y(x_i)$ is the ground-truth label of $x_i$. For each sample $x_i$, we first encode it using the pretrained model $f$ and then apply a linear classifier $g$ following that to generate the predicted distribution of different labels. For ease of theoretical analysis, we follow the settings of linear evaluation \citep{arora}, i.e., the pretrained encoder $f$ is frozen during the downstream evaluation process. Then we evaluate the prediction error
\begin{align*}
\mathcal{E}(f) = \E_{(x_i,y(x_i))}P(\hat{y}_i(x_i)\neq y(x_i)),    
\end{align*}
where $ \hat{y}_i(x_i)= {\rm argmax}(g(f(x_i)))$. 

\textbf{Content Generation.}
For the content generation tasks, we consider an unsupervised text dataset $\gD_u = \{x_i\}_{i=1}^M$ that consists of $M$ samples, we still assume that each $x$ consists of $s$ tokens, i.e., $x_i = (x_{i,1},\cdots,x_{i,s})$. The downstream objective is to evaluate the following likelihood with the pretrained model weights:
\begin{equation}
\label{equation-perplexity}
    \gL_{gen}(\Theta) = -\E_{x_i}\sum_{k}\log P(x_{i,k+1}|x_{i,1},\cdots,x_{i,k};\Theta),
\end{equation}
The exponential of this downstream objective is also called as perplexity, which is an important metric to evaluate the language model.

\section{A Theoretical Comparison between Autoregressive and Masked SSL}

\subsection{Generalization on Linear Classification}
In this section, we compare the generalization performance between autoregressive and masked SSL on linear classification tasks. To achieve this, we first establish a unified downstream classification guarantee of generative SSL and then compare autoregressive and masked SSL based on that.
\subsubsection{Downstream Classification Guarantees of Generative SSL}
Theorem \ref{thm:equivalence} establishes the mathematical equivalence between the generative SSL objectives and the asymmetric matrix decomposition objectives. Leveraging the Eckart-Young Theorem \citep{eckart1936approximation}, we can explicitly characterize the optimal solutions of the asymmetric decomposition objectives, which allows us to characterize the ideal features learned with generative SSL objectives. Substituting these pretrained features into the downstream classification objective yields the following theorem, which establishes the downstream guarantees for generative SSL. 
\begin{theorem}
We define the labeling error $\alpha$ as the probability that the conditional texts and the target text have different labels, i.e.,
\begin{align*}
    \alpha = \E_{(X_i,X_i^+)}\mathbbm{1}[y(X_i)\neq y(X_i^+)],
\end{align*}
where $y(\cdot)$ denotes the ground-truth label. Let $f^\star$ be the optimal solutions of $\gL_{SL}(f,W)$, then we obtain
\begin{align*}
    \E_{x_i} (\hat{y}(x_i) \neq y(x_i)) \leq c_1\sum\limits_{j={t+1}}^{N_D}{\sigma_{j}^4} +  c_2\alpha
\end{align*}
where $\sigma^2_{j}$ is the $j$-th largest singular value of the normalized co-occurrence matrix and $c_1,c_2$ are the constants.
\label{thm:generalization guarantee}
\end{theorem}
As shown in Theorem \ref{thm:generalization guarantee}, the downstream performance is mainly decided by two critical factors: the labeling error and the singular values of the co-occurrence matrix. In this paper, we follow the assumptions used in \citep{arora,wang2022chaos} that the labeling error between different tokens of the same samples is negligible. Consequently, our primary focus in this paper is on elucidating how the diverse objectives of autoregressive and masked SSL impact the singular values of the co-occurrence matrix.

A common way to understand the singular values is from a graph perspective. More precisely, the co-occurrence matrix $A$ can be viewed as the adjacent matrix of a bipartite graph $\gG$, where the nodes represent the conditional and target texts, and the edge weights denote the joint probability between them. The spectral graph theory \citep{chung1997spectral} states that the smaller singular values of the adjacent matrices correspond to stronger connectivity in the graphs, indicating properties such as fewer disconnected sub-graphs and shorter diameters. Consequently, Theorem \ref{thm:generalization guarantee} suggests that achieving superior downstream classification performance requires enhanced connectivity in the co-occurrence matrix.

This perspective provides insight into the differences between masked and autoregressive SSL models in downstream classification tasks. Intuitively, the multiple positions of target tokens in the masked SSL objectives have the potential to generate more connections than the single position in the autoregressive SSL objectives. In the following, we conduct both empirical and theoretical comparisons of the connectivity in autoregressive and masked co-occurrence matrices.

\subsubsection{Comparing Autoregressive and Masked SSL}

\begin{figure}[t]
 \centering
\includegraphics[width=.35\textwidth]{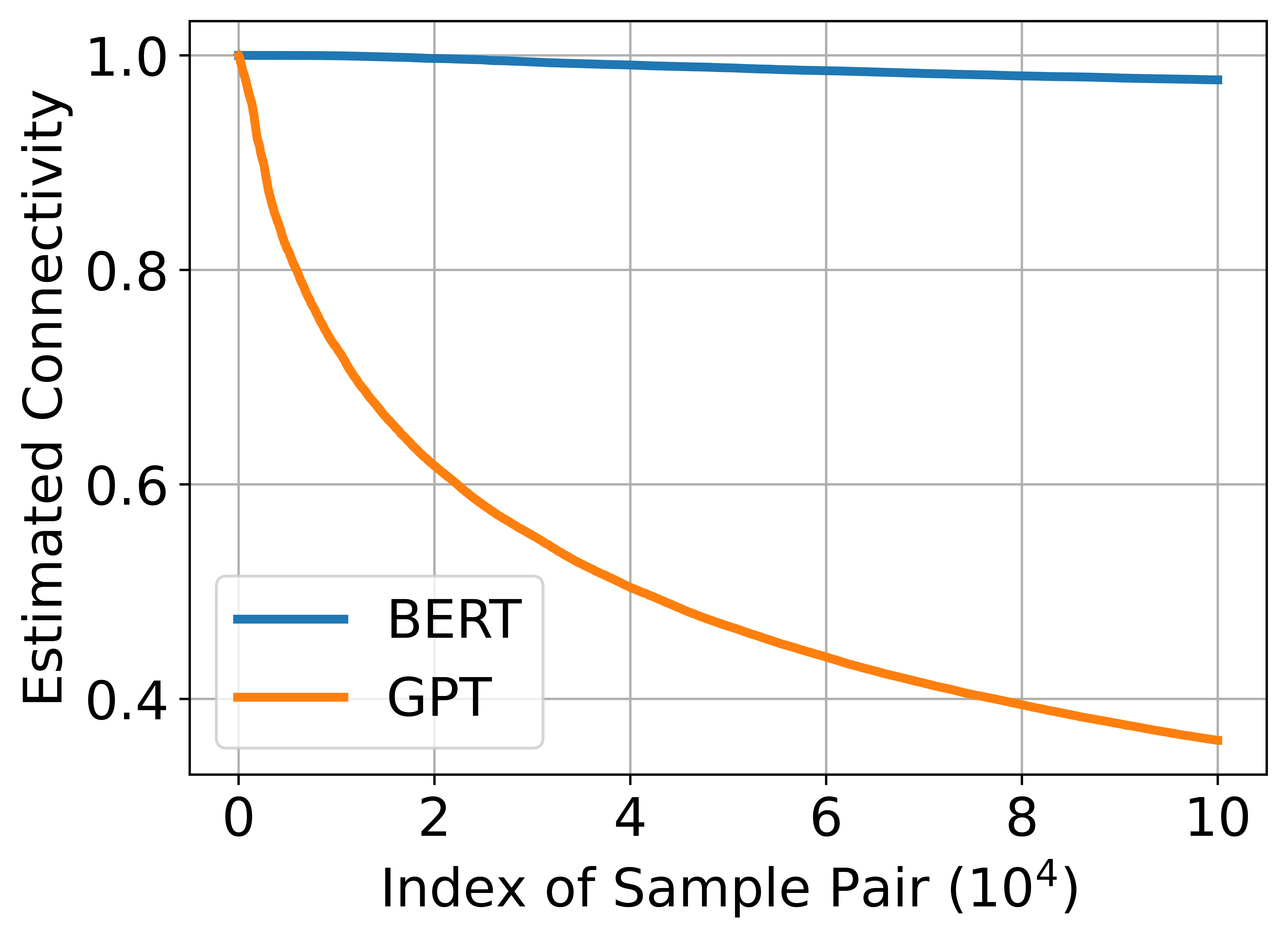}
    \vskip -0.1in
    \caption{Comparisons on estimated connectivity of the co-occurrence matrices of GPT and BERT. Details in Appendix \ref{appendix-estimated-detail}}
    \vskip -0.2in
\label{estimated connectivity}
\end{figure}
Based on the perspective above, we first empirically investigate the connectivity of autoregressive and masked co-occurrence matrices in the real-world dataset PILE\citep{gao2020pile}. As the co-occurrence matrices of the real-world datasets are not accessed, we calculate the average feature similarity between different texts as a surrogate metric for the sample connectivity. To ensure a fair comparison, we pretrain the same scale network with autoregressive and masked SSL objectives respectively. More details can be in Appendix. As shown in Figure \ref{estimated connectivity}, the estimated connectivity of the masked SSL co-occurrence matrices is significantly stronger than that of autoregressive SSL.

Besides, to theoretically compare the connectivity of autoregressive and masked SSL, we construct a toy model and calculate the singular values of the respective co-occurrence matrices.

\textbf{Toy model.} For the dataset $\gD_{sim}=\{(x_i,y(x_i))\}$, we denote that there exist $r$ classes in the dataset, i.e., $y(x_i) \in \{1,\cdots,r\}$. For the $k$-th token in the text $x_i$, we assume that $x_{i,k}$ are uniformly selected from the set $\{k\cdot T\cdot r+y_{i},k\cdot T\cdot r+y_{i}\cdot2,\cdots,k\cdot T\cdot r+y_{i}\cdot T\}$ with $T$ elements.
Then we can explicitly calculate the singular values of the co-occurrence matrix.
\begin{theorem}For the normalized co-occurrence matrix $\bar{A}, {A}'$ of autoregressive and masked SSL on the toy model dataset $\gD_{sim}$, when the length of masked texts $s\rho_m>1$, we obtain
$$
\begin{cases}
\sigma_j =\sigma'_j=1, &j \leq r,\\
\displaystyle \sigma_j =1 >  \sqrt{\frac{s(1-\rho_m)}{(s\rho_m)(s-1)}}=\sigma'_j,& r < j \leq r s, \\
\sigma_j = \sigma'_j=0,& j >rs, \\
\end{cases}
$$
where $\sigma_j,\sigma'_j$ are the $j$-th largest singular values of $\bar{A},\bar{A'}$, $r$ is the number of class, $s$ is the length of a sample and $\rho_m$ is the mask ratio. 
\label{thm:classification toy model}
\end{theorem}
According to Theorem \ref{thm:classification toy model}, the singular values of the masked SSL co-occurrence matrix are smaller than that of the autoregressive SSL, which verifies the intuition that the masked SSL can foster more connections by the multiple-lace predictions. Combined with Theorem \ref{thm:generalization guarantee}, the masked SSL objective has a superior downstream guarantee.  Additionally, it is noteworthy that the singular values of the masked SSL co-occurrence matrix decrease with a larger mask ratio. This observation implies that an aggressive mask ratio can effectively cluster more samples in the feature space, which is consistent with previous empirical findings in generative SSL \citep{mae,bert}.

In summary, both the empirical and theoretical analysis verify that multiple-place word predictions can bring more connections compared to the next word prediction, which suggests that the autoregressive SSL objective should introduce more diverse predictions to strengthen the connectivity between different texts. Naturally, we propose the following diversity-enhanced autoregressive SSL objective:
\vskip-15pt

\begin{small}
\begin{equation}
\begin{aligned}
\label{equation-diverse-auto-objective}
    \gL_{dar,t}(\Theta) 
    =-\E_{x_i}\sum_{k}\log P(x_{i,[k+1,k+t]}|x_{i,1},\cdots,x_{i,k};\Theta),
\end{aligned}
\end{equation}
\end{small}

where $x_{i,[k+1,k+t]}$ is a token randomly selected from $\{x_{i,k+1}, \cdots, x_{i,k+t}\},\ t\geq 1$. This objective lets the conditional sequence $x_{i,1},\cdots,x_{i,k+1}$ additionally predict $t-1$ more subsequent words, which introduces more diverse prediction targets and helps bring more connections. We will verify the effectiveness of this objective in Section \ref{section-experiments}.

\subsection{Generalization on Content Generation}

Besides the classification tasks we analyzed above, another important evaluation task for pretrained models is the content generation ability \citep{gpt,bert}. Previous works have empirically shown that the autoregressive SSL models outperform masked SSL models in downstream content generation tasks \citep{dong2019unified}. To further understand the advantages of autoregressive SSL models in content generation tasks, we first observe the generation ability of masked SSL models with different types of texts.

\begin{figure}
    \centering
\includegraphics[width=.35\textwidth]{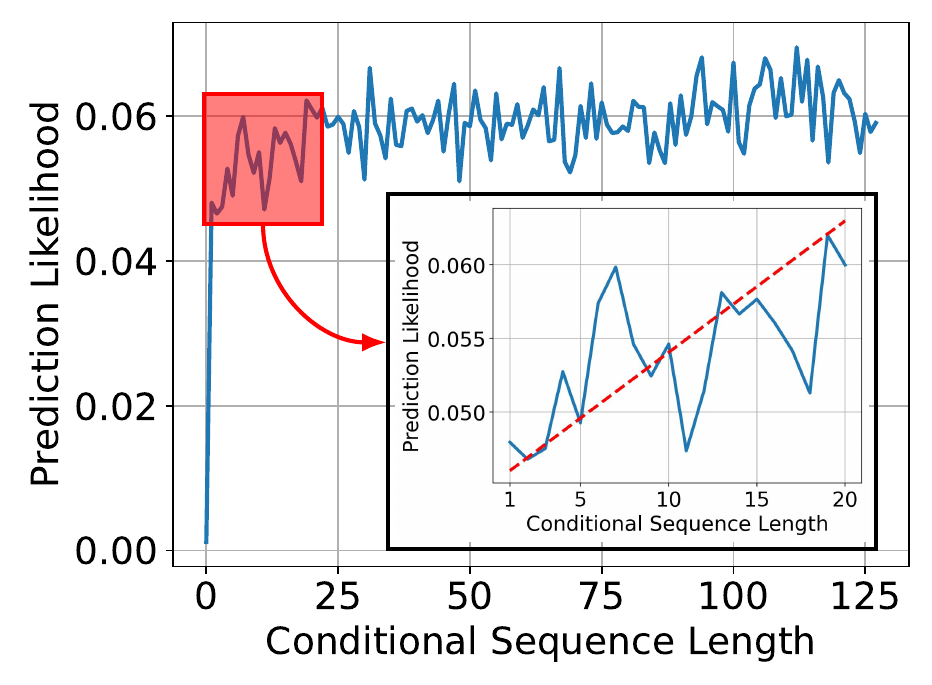}
    \caption{Comparison between different conditional sequence lengths in generation evaluation of masked SSL. Shorter conditional sequences suffer from low prediction likelihood.}
    \vskip -0.2in
    \label{fig:different_lengths}
\end{figure}

As shown in Figure \ref{fig:different_lengths}, we observe that the masked SSL model performs particularly unsatisfactory when generating from short texts. Intuitively, it can be attributed to a misalignment between the pretraining objective and the downstream evaluation tasks. To be specific, when predicting the masked words in the pretraining process, the lengths of inputs are fixed (e.g., 15\% of the texts). Consequently, the masked SSL model may struggle to accurately infer the complete texts due to limited information in the downstream generation tasks. In the subsequent analysis, we aim to theoretically substantiate this intuition. For the ease of our theoretical analysis, we consider using the linear attention as the pretrained encoder to compare the generation abilities of autoregressive and masked SSL models.

\textbf{Linear Attention.} The general form of linear attention is given by:
\begin{equation*}
f(x_{i})=x_{i}W^{Q}(x_{i}W^{K})^{T}x_{i}W^{V}=QK^{T}V,
\end{equation*}
where $W^{Q},W^{K},W^{V}$ are projections.

By comparing the evaluation error on the downstream generation tasks, we obtain the following theorem.
\begin{theorem}
Let $f_{mask}$ be the model pretrained by masked SSL. We establish the upper bound of masked SSL in content generation tasks:
\begin{align*}
\mathcal{L}_{gen}(f_{mask})\leq \frac{\sum_{k}\left(\frac{w_{k}^{2}}{(k-1)^{6}}+w_{k}\left\lVert W \right\rVert _{2}^2 \eta\right)}{2s(1-\rho _{m})}+\delta _{mask} + 1,
\end{align*}
where $s$ is the length of a sequence, $\rho_m$ is the mask ratio, $w_{k}=((s(1-\rho _{m}))^{3}-(k-1)^{3})$ is the misalignment degree of input lengths, $\eta = \max\Vert x_{i,a}W^{Q}(x_{i,b}W^{K})^{\top }x_{i,c}W^{V}-x_{i,\alpha}W^{Q}(x_{i,\beta}W^{K})^{\top }x_{i,\gamma}W^{V}\Vert$ represents the difference between pretrained model outputs of different positions, and $\delta_{mask}=\max(-(Wf(X_i))^\top\mathbbm{1}_{X_i^+} + ((Wf(X_i))^\top \mathbbm{1}_{X_i^-})^2)$ represents the pretraining error of masked SSL.

For the autoregressive SSL, we obtain
\begin{align*}
\mathcal{L}_{gen}(f_{ar}) \leq \delta_{ar}.
\end{align*}
 \textbf{Comparing two upper bounds, the gap between masked and autoregressive SSL is:}
\begin{equation}
\frac{\sum_{k}^{}\left(\frac{w_{k}^{2}}{(k-1)^{6}}+w_{k}\left\lVert W \right\rVert _{2}^2 \eta\right)}{2s(1-\rho _{m})}+1+\delta_{mask}-\delta_{ar}.
\label{eqn:generation gap}
\end{equation}
Consequently, autoregressive SSL obtains a smaller error when the pretraining errors are negligible.
\label{thm:length mismatch}
\end{theorem}

As stated in Theorem \ref{thm:length mismatch}, besides the canonical pretraining error, the autoregressive SSL models have superior performance guarantees in downstream content generation tasks compared to masked SSL models, and the performance gap is mainly decided by two crucial factors: the alignment of input lengths ($w_k$) and the consistency of different positions ($\eta$). In the following, we respectively discuss these two factors.

\textbf{Alignment of input lengths.}  We note that when the lengths of test samples are close to the length of unmasked texts in the pretraining objective, the performance of masked SSL models can approach autoregressive SSL models, which proves the length misalignment between unmasked texts in pretraining and test examples in downstream is a crucial reason for the inferior performance of masked SSL models. 

\textbf{Consistency of different positions.} The term $\eta$ in Equation (\ref{eqn:generation gap}) evaluates the difference in model output across various positions in a sequence. When the output distribution is more consistent, the performance of the masked SSL model is better. Consequently, it is advantageous to encourage the model to generate predictions based on the entire sentence rather than focusing on specific tokens. Besides, the consistency of different positions offers some other potential benefits. For example, by uniformly considering the semantic information across different positions in a sentence, the model can avoid generating shortcuts and overfit solutions.

In summary, the theoretical bounds in Theorem \ref{thm:length mismatch} provide two principled guidelines to improve the content generation performance of masked SSL: (1) we should try to mitigate the length misalignment between pretraining and downstream examples, and (2) the model predictions should be encouraged to be consistent with different tokens in the same sentences. 

Inspired by the theoretical analysis, predicting the masked texts with different lengths of unmasked texts is a straightforward solution to mitigate the misalignment and improve the generation performance of masked SSL. Consequently, we propose the following variable-length masked objective:
\vspace{-0.1in}
\begin{equation}
\label{equation-variable-length-objective}
\begin{aligned}
&\gL_{vm}(\Theta)\\ &= -\E_{\rho_m}\E_{(x_i^1,x_i^2)}\sum_{k}\log P(x^2_{i,k}|x^1_{i,1},\cdots,x^1_{i,s(1-\rho_m)};\Theta).   
\end{aligned}
\end{equation}
In this objective, the mask ratio $\rho_m$ is not fixed. Instead, it is randomly sampled from a range. The variable mask ratio has the potential to mitigate the length misalignment between pretraining and downstream examples, improving the content generation performance. We will verify its effectiveness in Section \ref{section-experiments}.

\begin{table*}[t]
  \centering
  \caption{GLUE test set results of autoregressive objective and diversity-enhanced autoregressive objective scored with 5 epochs fine-tuning on each test set. We report F1 scores for QQP and MRPC, Pearson correlation coefficient for STS-B, Matthews correlation coefficient for CoLA and accuracy scores for the other tasks.}
  \vskip 0.1in
  \begin{tabular}{lccccccccc}
    \toprule
    Objective & MNLI & SST-2 & STSB & RTE & QNLI & QQP & MRPC & CoLA & \textbf{Avg} \\
    \midrule
    Autoregressive & 78.2/79.3 & \textbf{90.3} & 82.1 & 54.3 & 84.7 & 85.6 & 81.9 & 30.9 & 74.1 \\
    Diversity-enhanced Autoregressive & \textbf{80.5}/\textbf{80.6} & 89.9 & \textbf{83.7} & \textbf{56.4} & \textbf{87.2} & \textbf{86.3} & \textbf{85.1} & \textbf{34.2} & \textbf{76.0} \\
    \bottomrule
  \end{tabular}
  \label{table-diversity-enhanced-language}
  \vspace{-.2in}
\end{table*}

\subsection{Discussion}

In summary, this section introduces the first theoretical comparison between autoregressive and masked SSL. As autoregressive and masked SSL exhibit contrasting advantages in downstream classification and content generation tasks \citep{gpt,bert}, the two most crucial evaluation tasks for pretrained models seem to be contradictory. However, by revealing the limitations of autoregressive SSL in classification tasks (fixed position of target tokens) and limitations of masked SSL in content generation tasks (fixed length of unmasked texts), we note that the limitations of these two paradigms are not contradictory to each other. Instead, they deliver a consistent message: we should encourage diversity (in inputs, prediction objectives, etc) to the generative SSL objectives to improve the generalization performance. As a verification of our analysis, we will show the performance of our proposed objectives across both language and vision tasks in the following sections.

\section{Experiments}
\label{section-experiments}
In this section, we will verify the effectiveness of our new proposed objectives presented in Equation (\ref{equation-diverse-auto-objective}) and Equation (\ref{equation-variable-length-objective}). We will conduct experiments on both vision and language tasks to demonstrate the generality of our methods.

\subsection{Diversity-enhanced Autoregressive Objective Improves Classification Ability}
In this part, we will consider the diversity-enhanced autoregressive objective as proposed in Equation (\ref{equation-diverse-auto-objective}). The diversity-enhanced autoregressive objective is to predict the next $t$ tokens of conditional texts. Inspired by \citep{wang2018semi}, which predicts several tokens simultaneously, we perform group-autoregressive modeling on the sequence to effectively realize this objective. Specifically, the sample sequence $(x_{i,1},\dots,x_{i,s})$ is divided into several consecutive groups $G_{i,1},\dots,G_{i,l}$ in order, with each group containing any number of tokens. The prediction is performed group by group, i.e., within each group, the tokens are predicted in parallel, while across group, the predictions are sequential. In this way, given $G_{i,1},\dots,G_{i,s}$, the model is to predict all tokens from $G_{i,s+1}$. Therefore, by setting $|G_{i,2}|=\dots=|G_{i,l-1}|=t$ and summing $|G_{i,1}|$ from $1$ to $t$, the semi-autoregressive modeling task
\begin{equation}
\label{equation-semi-autoregressive}
    \begin{aligned}
        &\mathcal{L}_{semi,t}(\Theta)\\
        &=-\E_{x_i}\sum_{|G_1|}\sum_{s}\log P(G_{i,s+1}|G_{i,1},\dots,G_{i,s}; \Theta)
    \end{aligned}
\end{equation}
is equivalent to Equation (\ref{equation-diverse-auto-objective}). In our experiments, we set $t=2$ and uniformly sampling $|G_{i,1}|$ from $\{1,2,\dots,t\}$ instead. Since there exist multiple prediction targets for one condition sequence, it is difficult for vanilla Transformer \citep{vaswani2023attention} to model. Therefore, for language tasks, we will leverage and extend the two-stream attention module proposed in XLNet \citep{yang2020xlnet} by designing causal masks for semi-autoregressive modeling, where the two-stream attention can be used to model more various dependency relationship between tokens compared to vanilla Transformer. More details on the realization will be illustrated in Appendix \ref{appendix-two-stream}. Similarly, for vision tasks, we use the two-stream attention version of ViT \citep{dosovitskiy2020image}, which is also adopted in the previous work \citep{hua2022self}.

\begin{table}[t]
\centering
    \caption{Experiment results of autoregressive objective and diversity-enhanced autoregressive objective on image classification tasks. ViT-S(mall) is trained on ImageNet-100 and ViT-B(ase) is trained on ImageNet-1K. LP ACC. refers to linear probing accuracy (\%). FT Acc. refers to fine-tuning accuracy (\%).}
    \vskip 0.1in
        \begin{tabular}{lccc}
        \toprule
        Arch. & Objective & LP Acc. & FT Acc. \\
        \midrule
        \multirow{2}{*}{ViT-S} & Autoregressive & 33.1 & 81.1 \\
            &Diversity-enhanced & \textbf{36.2} & \textbf{83.2} \\
        \midrule
        \multirow{2}{*}{ViT-B} & Autoregressive & 56.2 & 82.5\\
            &Diversity-enhanced & \textbf{59.4} & \textbf{82.9} \\
        \bottomrule
        \end{tabular}
        \label{table-diversity-enhanced-vision}
\end{table}

\textbf{Language Tasks.} The model is pretrained on the Pile dataset \citep{gao2020pile}, which contains content from 22 diverse sources. We use a decoder-only Transformer with 16 layers and hidden size of 768. We train the model for 100K steps with batch size of 8192. The other pretraining procedure follows the protocol proposed by Cramming \citep{geiping2023cramming}. The model is then finetuned with 5 epochs and evaluated on the General Language Understanding Evaluation (GLUE) benchmark \citep{wang2018glue}, which is a collection of several language understanding tasks.

Table \ref{table-diversity-enhanced-language} presents the GLUE test results. These results reveal that the model trained with the diversity-enhanced autoregressive objective consistently exhibits improvement in classification tasks. Our proposed objective has a gain of 1.9\% on the average score and 3.3\% on the hardest task CoLA. These improvements support the effectiveness of our theoretical findings regarding the classification ability of autoregressive models.

\textbf{Vision Tasks.} The model is pretrained on ImageNet-1K with ViT-Base and ImageNet-100 with ViT-Small \citep{imagenet}. The training epoch is 200 with a warm-up of 10 epochs. The batch size is set to 4096 following previous works \citep{mae,hua2022self,xie2022simmim}. We set the base learning rate as 5e-4 and use Adam optimizer \citep{kingma2014adam}. After the pretraining is finished, we
perform downstream tasks under two protocols: For linear probing, we train a linear classifier on the frozen pretrained encoder for 90 epochs; For non-linear fine-tuning, we train both the pretrained encoder and the linear classifier with the soft target cross entropy loss \citep{peterson2019human} for 100 epochs. The batch size is set to 4096.

The experiment results are presented in Table \ref{table-diversity-enhanced-vision}. The diversity-enhanced autoregressive objective improves linear probing accuracy by $\sim$3\%. It also improves fine-tuning accuracy on ViT-S by 2.1\% and on ViT-B by 0.4\%. These positive outcomes affirm the effectiveness of our proposed objective, emphasizing our theoretical insight that it is beneficial to have more connections in generative SSL.

We also conduct various extensive experiments on vision tasks including few-shot learning as \citet{dubois2023evaluating}, multi-epoch training as \citet{li2023architecture} and transfer learning as \citet{kong2023understanding}.

\textbf{Few-shot learning.} We conduct experiments with the diversity-enhanced autoregressive objective. In the stage of the downstream tasks, we only use 1\% of the label. Finetuning and linear probing results are shown in Table \ref{tab:few-shot}. As shown in the above table, the diversity-enhanced autoregressive objective improves both the few-shot linear and finetuning accuracy, which further verifies the effectiveness of our proposed objectives across different tasks.
\vspace{-0.2in}
\begin{table}[H]
    \centering
    \caption{Few-shot evaluation accuracy (\%) on ImageNet-1K with ViT Base. Only 10\% of the training data is given during training.}
    \vspace{0.1in}
    \begin{tabular}{lcc}
    \toprule
        Objective & LP Acc. & FT Acc. \\
    \midrule
        Autoregressive & 41.2 & 70.2 \\
        Diversity-enchanced & \textbf{45.3} & \textbf{71.1} \\
    \bottomrule
    \end{tabular}
    \label{tab:few-shot}
    \vspace{-0.1in}
\end{table}

\textbf{Multi-epoch training.} We conduct experiments with the diversity-enhanced autoregressive objective. We respectively pretrain the models with 100 epochs and 200 epochs. Finetuning and linear probing results are shown in Table \ref{tab:multi-row}. As shown in the above table, the diversity-enhanced autoregressive models with different training epochs consistently achieve significant improvement in classification tasks, which also verifies the effectiveness of our proposed objectives.
\vspace{-0.2in}
\begin{table}[H]
    \centering
    \caption{Test accuracy (\%) with multi-epoch training on ImageNet-1K with ViT Base.}
    \vspace{0.1in}
    \begin{tabular}{clcc}
    \toprule
        Epochs & Objective & LP Acc. & FT Acc. \\
    \midrule
        \multirow{2}{*}{100} & Autoregressive & 48.2 & 81.4 \\
        &Diversity-enchanced & \textbf{52.6} & \textbf{82.1} \\
    \midrule
        \multirow{2}{*}{200} & Autoregressive & 56.2 & 82.5 \\
        &Diversity-enchanced & \textbf{59.4} & \textbf{82.9} \\
    \bottomrule    
    \end{tabular} 
    \label{tab:multi-row}
    \vspace{-0.1in}
\end{table}

\textbf{Transfer learning.} We conduct experiments with the diversity-enhanced autoregressive objective. We evaluate the transfer learning performance pretrained on ImageNet-1K on 9 downstream datasets as \citet{zhao2022arcl} and \citet{ericsson2021well}, which are FGVC Aircraft, Caltech-101, Stanford Cars, CIFAR10, CIFAR100, DTD, Oxford 102 Flowers, Food-101 and Oxford-IIIT Pets. For linear evaluation, multinomial logistic regression is fit on the extracted features. Results are shown in Table \ref{tab:transfer}. As shown in the table, the diversity-enhanced autoregressive models also show superior performance in transfer learning, which further verifies the generalization performance of our models.
\vskip -0.2in
\begin{table}[H]
    \centering
    \caption{Transfer learning average accuracy (\%) on 9 downstream datasets  with ViT Base.}
    \vskip 0.1in
    \begin{tabular}{lc}
    \toprule
        Objective & Average accuracy \\
    \midrule
        Autoregressive & 76.2 \\
        Diversity-enchanced & \textbf{78.9} \\
    \bottomrule
    \end{tabular}
    \label{tab:transfer}
    \vspace{-0.1in}
\end{table}

The experiments across three tasks and multiple datasets show that the modified objectives consistently improve the performance of downstream tasks. This proves the superiority of our objective and verifies the effectiveness of our proposed theory.

\begin{table*}[h]  
    \begin{minipage}{.48\textwidth}  
    \centering
    \caption{Perplexity of models trained with masked objective and variable-length masked objective on the Pile dataset and C4 dataset. Smaller perplexity indicates better generation ability. $[0.15,0.3]$ means uniformly sampling from $[0.15,0.3]$ during training.}
    \vskip 0.1in
        \begin{tabular}{lcc}
        \toprule
        Objective (Mask ratio) & Pile ($\downarrow$) & C4 ($\downarrow$) \\
        \midrule
        Masked ($0.15$) & 59.6 & 71.2\\
        Masked ($0.3$) & 50.2 & 63.8\\
        Variable-length ($[0.15,0.3]$) & \textbf{45.1} & \textbf{59.4}\\
        \bottomrule
        \end{tabular}
        \label{table-perplexity}
    \end{minipage}%
    \hfill
    \begin{minipage}{.48\textwidth}  %
    \centering
    \caption{Reconstruction ability of models trained with masked objective and variable-length masked objective on ImageNet dataset. There are two evaluation metrics: L2 loss and Perceptual loss. $[0.5,0.75]$ means uniformly sampling from $[0.5,0.75]$ during training.}
    \vskip 0.1in
        \begin{tabular}{lcc}
        \toprule
        Objective (Mask ratio) & L2 ($\downarrow$) & Perceptual ($\downarrow$)\\
        \midrule
        Masked ($0.75$) & 0.127 & 0.201\\
        Masked ($0.5$) & 0.115 & 0.192\\
        Variable-length ($[0.5,0.75]$) & \textbf{0.072} & \textbf{0.136}\\
        \bottomrule
        \end{tabular}
        \label{table-reconstruction-loss}
    \end{minipage}
\vspace{-0.15in}    
\end{table*}

\vspace{-0.1in}
\subsection{Variable-length Masked Objective Improves Generation Ability}
In this part, we will consider the variable-length MIM objective as proposed in Equation (\ref{equation-variable-length-objective}). Recall that the variable-length MIM objective leverages variable mask ratio to mitigate the length misalignment problem in masked models. We will consider different sampling strategies for the mask ratio in language tasks and vision tasks.

\textbf{Language Tasks.} Similar to the experiments on diversity-enhanced autoregressive Objective, the model is pretrained on the Pile dataset \citep{gao2020pile}. We use a decoder-only Transformer (without causal mask) with 16 layers and hidden size of 768. The pretraining procedure follows the protocol proposed by Cramming \citep{geiping2023cramming}. Regarding the mask ratio in the masked model, we explore three options: (1) retaining the original BERT value of 0.15, (2) opting for a larger ratio of 0.3, and (3) uniformly sampling from the range $[0.15, 0.3]$, aligning with our proposed variable-length objective. We adopt perplexity as the evaluation metric, which is the exponential of Equation (\ref{equation-perplexity}). We will evaluate the model on two datasets: (1) the tail of the Pile dataset, which is never exhibited to the model during the pretraining process; (2) a subset of C4 dataset \citep{2019t5}, which is a colossal, cleaned version of Common Crawl's web crawl corpus.

Table \ref{table-perplexity} demonstrates that the model trained with the variable-length objective exhibits significantly lower perplexity compared to the masked models trained with a fixed mask ratio on both datasets. This suggests that the variable-length objective can improve the generation ability of the model. These enhancements support our theoretical insights into the generation ability of masked models.

\textbf{Vision Tasks.} We select MAE \citep{mae} as the baseline model, with pretraining hyper-parameters consistent with those in the diversity-enhanced experiments. After the pretraining is finished, we perform image reconstruction tasks. The image is first randomly masked out with 25\% portion. Then the model predicts the masked part token by token in an autoregressive way. Reconstruction quality is evaluated using two metrics between the reconstructed and original images: L2 loss at the pixel level and LPIPS loss \citep{zhang2018perceptual} at the high-level feature level. The average loss is calculated over the ImageNet validation set.

Table \ref{table-reconstruction-loss} reveals that the model trained with the variable-length objective experiences significantly lower (approximately 30\%) reconstruction loss compared to the masked models trained with a fixed mask ratio. These enhancements underscore our theoretical insight that mitigating misalignment of unmasked lengths between pretraining and test examples is crucial for generation ability.

\section{Conclusion}
In this paper, we propose the first theoretical comparison between two primary generative self-supervised (SSL) paradigms: autoregressive SSL and masked SSL. In particular, we establish the theoretical guarantees for autoregressive SSL and masked SSL in the most common evaluation tasks, i.e., classification and content generation. Through empirical and theoretical analyses, we delineate the advantages of masked SSL in classification tasks and the advantages of autoregressive SSL in content generation tasks. Building upon the insights from our theoretical analysis, we formulate principled guidelines for generative SSL and introduce two improved generative SSL objectives. Empirically, we verify that our proposed objectives significantly improve the performance of generative SSL models accoss both visual and language datasets.

\section*{Acknowledgements}
Yisen Wang was supported by National Key R\&D Program of China (2022ZD0160300), National Natural Science Foundation of China (92370129, 62376010), Beijing Nova Program (20230484344), and CCF-Baichuan-EB Fund.

\section*{Impact Statement}
This paper presents work whose goal is to advance the field of Machine Learning. There are many potential societal consequences of our work, none of which we feel must be specifically highlighted here.

\bibliography{example_paper}
\bibliographystyle{icml2024}

\newpage
\appendix
\onecolumn

\section{Proofs}

\subsection{Proof of Theorem \ref{thm:equivalence}}
\begin{proof}
 Expanding the decomposition object and we obtain,
\begin{align*}
\Vert \bar{A} - FW'^\top \Vert ^2
        &=\sum\limits_{X_i,X_i^+}\left(\bar{A}_{X_i,X_i^+}-F_{X_i}(W'_{X_i^+})^\top\right)^2\\
        &=\sum\limits_{X_i,X_i^+}\left( \frac{P_M(X_i,X_i^+)}{\sqrt{P_C(X_i)P_G(X_i^+)}}-\sqrt{P_C(X_i)}f(X_i)^\top \sqrt{P_G(X_i^+)}(W_{X_{i}^+})^\top \right)^2\\
        &=\sum\limits_{X_i,X_i^+}\left(\frac{P_M(X_i,X_i^+)^2}{P_C(X_i)P_G(X_i^+)} +P_C(X_i)P_G(X_i^+)\left(f(X_i)^\top (W_{X_{i}^+})^\top\right)^2 -2P_M(X_i,X_i^+)f(X_i)^\top (W_{X_{i}^+})^\top\right)\\
        &=\sum\limits_{X_i,X_i^+}\left(\frac{P_M(X_i,X_i^+)^2}{P_C(X_i)P_G(X_i^+)} \right) 
        -2\E_{(X_i,X_i^+)} f(X_i)^\top (W_{X_{i}^+})^\top + \E_{X_i, X_i^-}\left(f(X_i)^\top (W_{X_{i}^-})^\top \right)^2\\
        &= \sum\limits_{X_i,X_i^+}\left(\frac{P_M(X_i,X_i^+)^2}{P_C(X_i)P_G(X_i^+)} \right) -\E_{(X_i,X_i^+)}  (Wf(X_i))^\top\mathbbm{1}_{X_i^+} +\E_{X_i,X_i^-} ((Wf(X_i))^\top \mathbbm{1}_{X_i^-})^2\\
        &= \gL_{SL}(f,W) + \mathrm{const}.
\end{align*}    
\end{proof}

\subsection{{Proofs of Theorem \ref{thm:generalization guarantee}}}
We first introduce a lemma that will be used in the following proofs.
\begin{lemma}[Lemma 3.1 in \citep{haochen}]
For two learned embedding matrices $F$, $\widetilde{F}$, a diagonal matrix $D$ and an invertible matrix $R$, if $F = D \widetilde{F} R$, they have the equal linear probing error, i.e.,
\begin{align*}
\mathcal{E} (F) = \mathcal{E}(\widetilde{F}).
\end{align*}
\label{lem:linear absorbed}
\end{lemma}
\begin{proof}
    The proofs mainly follow the theoretical framework proposed in \citep{zhang2023generalization}.
    According to Eckart-Young Theorem \citep{eckart1936approximation}, the optimal solution $F^\star, (W')^\star$ of the decomposition objective $\gL_{MF}(F,W') = \Vert \bar{A} -FW'^\top \Vert ^2$ satisfy:
    \begin{align*}
    F^\star (W'^\star)^\top = U^t \operatorname{diag}(\sigma_1,...,\sigma_t)(V^t)^\top,
    \end{align*}
    where we denote $\bar{A}=U\Sigma V^\top$ as the singular value decomposition of $\bar{A}$, $(\sigma_1,...,\sigma_t)$ are the $t$-largest singular values of $\bar{A}$, the $t$-th column of $U^t \in \mathbb{R}^{N \times t}$ contains the corresponding eigenvectors of the $t$-th largest singular values and $V^t  \in \mathbb{R}^{N_D \times t}$ is a unitary matrix. Then we respectively obtain the optimal solutions $F^\star$ and $(W')^\star$:
    \begin{align*}
    F^\star &= U^t D R,\\
    (W')^\star &= V^t \operatorname{diag} (\sigma_1,...,\sigma_t)D^{-1} R,
    \end{align*}
where $R \in \mathbb{R}^{t\times t}$ is a unitary matrix and $D$ is an invertible diagonal matrix.
Then we define a symmetric spectral loss:
\begin{equation}
    \gL_{SCL}(f) = -\E_{(X_i,X_i')}  f(X_i)^\top f(X'_i)+\E_{X_i,(X_i')^-} (f(X_i)^\top f((X_i')^-))^2,
\end{equation}
where $(X_i,X_i')\sim P_T, P_T(X_i,X_i') = \E_{X_i^+}P_M(X_i|X_i^+)P_M(X_i'|X_i^+)$, and $X_i, X_i' \sim P_C(X_i)$.
Following the proof of theorem \ref{thm:equivalence}, the symmetric spectral loss is also equivalent to a matrix decomposition loss, i.e., $\gL_{ SCL}(f) = \Vert \tilde{P}_T - FF^\top \Vert ^2 +const$, where $(\tilde{P}_T)_{(X_i,X_i')} = \frac{P_T(X_i, X_i')}{\sqrt{\gP_C(X_i)\gP_C(X_i')}} $ and $(F)_{X_i} = \frac{f(X_i)^\top}{\sqrt{P_C(X_i)}}$. Then we consider the objective $\Vert \tilde{P}_T-FF^\top \Vert ^2 $. Similar to the asymmetric decomposition objective, the optimal solution can be represented as:
\begin{align*}
(F^\star)' = U_T^t D_T R_T,
\end{align*}
where $U_T^t \in \mathbb{R}^{N\times t}$ contains $t$ corresponding eigenvectors of $t$ largest singular values of $\tilde{P}_T$, $D_T\in \mathbb{R}^{t\times t}$ is an invertible diagonal matrix and $R_T\in \mathbb{R}^{t\times t}$ is a unitary matrix. In the next step, we analyze the relationship between $\bar{A}$ and $\tilde{P}_T$. Considering the $(X_i,X_i')$-th element of $\bar{A}\bar{A}^\top$, we have
    \begin{align*}
    (\bar{A}\bar{A}^\top)_{X_i,X_i'} &= \sum\limits_{x_i^+} (\bar{A})_{X_i,X_i^+} (\bar{A})_{X_i',X_i^+}\\
    &=\sum\limits_{X_i^+} \frac{\gP_M(X_i,X_i^+)\gP_M(X_i',X_i^+)}{\gP_G(X_i^+)\sqrt{\gP_C(X_i)\gP_C(X_i')}}\\
    &=\frac{1}{\sqrt{\gP_C(X_i)\gP_C(X_i')}} \sum\limits_{X_i^+} \gP_G(X_i^+)\gP_M(X_i|X_i^+)\gP_M(X_i'|X_i^+)
    &\left(\gP_M(X_i,X_i^+) = \gP_M(X_i|X_i^+)\gP_G(X_i^+)\right)\\
    &=\frac{\E_{X_i^+}\gP_M(X_i|X_i^+)\gP_M(X_i'|X_i^+)  }{\sqrt{\gP_C(X_i)\gP_C(X_i')}}\\
    &= (\tilde{P}_T)_{X_i,X_i'}.
    \end{align*}
We know that $\tilde{P}_T = \bar{A}\bar{A}^\top$, so $\tilde{P}_T$ and $\bar{A}$ share the same eigenvectors, i.e., $U^t= U_T^t$. As $D, D_2, R, D_T, R_T$ are invertible matrices and the product of the invertible matrices is still invertible, we obtain
\begin{align*}
    F ^\star =  (F ^\star)' T,
\end{align*}
where $T = (D_T)^{-1} (R_T)^{-1} D R$ is an invertible matrix. 
With Lemma \ref{lem:linear absorbed}, we obtain
\begin{align*}
   \mathcal{E}(f^\star) = \mathcal{E}(f^\star_{SCL}). 
\end{align*}
\end{proof}
Then we introduce another lemma:
\begin{lemma}[Theorem 5.1 in \citep{zhang2023identifiable}]
 We denote $\tilde{\alpha}$ as the probability that the conditional texts have different labels, i.e., $\tilde{\alpha} = \E _{(X_i,X_i')\sim P_T} \mathbbm{1}[y(X_i) \neq y(X_i')]$. Let $f^\star_{SCL}$ be the optimal solutions of the symmetric spectral objective Then, we have the downstream classification guarantee:
 \begin{equation}
\begin{aligned}
\mathcal{E}(f^\star_{SCL}) \leq c_1\sum \limits_{i=t+1}^{N_D} \tilde{\sigma}_i ^2+c_2\cdot\tilde{\alpha},
\end{aligned}
\label{eqn:guarantee of SCL}
\end{equation}
where $\tilde{\sigma}_i$ is the $i$-th largest eigenvalue of $\tilde{P}_T$, and $c_1,c_2$ are constants.
\label{lemma:guarantee of SCL}
\end{lemma}
And we continue the proofs.
\begin{proof}
Combined with Lemma \ref{lemma:guarantee of SCL}, we obtain
\begin{align*}
   \mathcal{E}(f^\star) = \mathcal{E}(f_{SCL}^\star) \leq  c_1\sum \limits_{i=t+1}^{N_D} \tilde{\sigma}_i ^2+c_2\cdot\tilde{\alpha}. 
\end{align*}
As $\tilde{P}_T = \bar{A}\bar{A}^\top$, we obtain $\tilde{\sigma} = \sigma^2$. And for the labeling error, we have
\begin{align*}
\tilde{\alpha} &= \sum\limits_{X_i,X'_i} \gP_T(X_i,X_i')\mathbbm{1}[y(X_i)\neq y(X'_i)]\\
&=\sum\limits_{X_i,X'_i} \E_{X_i^+}\left[\gP_M(X_i|X_i^+)\gP_M(X_i|X_i^+)\mathbbm{1}[y(X_i)\neq y(X'_i)]\right]\\
&\leq\sum\limits_{X_i,X'_i} \E_{X_i^+}\left[\gP_M(X_i|X_i^+)\gP_M(X_i|X_i^+)(\mathbbm{1}[y(X_i)\neq y(X_i^+)]+\mathbbm{1}[y(X'_i)\neq y(X_i^+)])\right]\\
&= 2 \E_{X_i^+} [\gP_M(X_i|X_i^+)\mathbbm{1}[y(X_i)\neq y(X_i^+)]]\\
&= 2 \E_{X_i,X_i^+} \mathbbm{1}[y(X_i)\neq y(X_i^+)]\\
&= 2\alpha.
\end{align*}
So we obtain 
\begin{equation}
   \mathcal{E}(f^\star) \leq  c_1\sum \limits_{i=t+1}^{N_D} \sigma_i ^4+c_2'\cdot\alpha,
\end{equation}
which finishes the proofs of Theorem \ref{thm:generalization guarantee}.
\end{proof}

\subsection{Proofs of Theorem \ref{thm:classification toy model}}

We first introduce a lemma that will be used in the following proofs.
\begin{lemma}[Theorem 4.2 in \citep{zhang2023generalization}]
For a block matrix:
$$ P = \begin{pmatrix}
p_a & \cdots & p_a & p_b & \cdots & p_b & \cdots & p_b & \cdots & p_b \\
\cdots & \cdots & \cdots & \cdots & \cdots & \cdots & \cdots & \cdots & \cdots & \cdots \\
p_a & \cdots & p_a & p_b & \cdots & p_b & \cdots & p_b & \cdots & p_b \\
p_b & \cdots & p_b & p_a & \cdots & p_a & \cdots & p_b & \cdots & p_b \\
\cdots & \cdots & \cdots & \cdots & \cdots & \cdots & \cdots & \cdots & \cdots & \cdots \\
p_b & \cdots & p_b & p_a & \cdots & p_a & \cdots & p_b & \cdots & p_b \\
\cdots & \cdots & \cdots & \cdots & \cdots & \cdots & \cdots & \cdots & \cdots & \cdots \\
p_b & \cdots & \cdots & \cdots & \cdots & p_b & \cdots & p_a & \cdots & p_a \\
\cdots & \cdots & \cdots & \cdots & \cdots & \cdots & \cdots & \cdots & \cdots & \cdots \\
p_b & \cdots & \cdots & \cdots & \cdots & p_b & \cdots & p_a & \cdots & p_a 
\end{pmatrix} ,$$
each row and each columns has $s_a \cdot s_b$ elements. Among them, $s_a$ elements are $p_a$ and else are $p_b$. Then the singular values of it are:
    \begin{align*}
        &\sigma_1 = s_a\cdot p_a+(s_{b}-1)\cdot s_a\cdot p_b ,\\
        &\sigma_2 = \cdots = \sigma_{s_b} = s_a\cdot\vert p_b - p_a\vert,\\
        &\sigma_{s_a+1} = \cdots = \sigma_{s_a\cdot s_b} = 0.
    \end{align*}
\label{lem:block matrix}
\end{lemma}

\begin{proof}
For the autoregressive objective, the co-occurrence matrix size is $( \frac{T-T^s}{1-T} r \times ST r)$.

As the tokens are uniformly selected, we know that $A_{X,X^+} = \frac{1}{rsT^{i+1}}$ (when $i\neq 0)$, where $i$ is the length of the condition texts $X$. Consequently, the elements in the normalized $\bar{A}$ satisfy $\bar{A}_{X,X^+}=\frac{1}{\sqrt{T^{i+1}}}$.

As the samples of different classes are disconnected, we only need to consider the singular values of the sub-matrix of intra-class samples. We note that each sub-matrix has $s$ non-zero diagonal blocks. The $i$-th block size is $(T^i \times T^i)$ and the elements in the same block are equal to $\frac{1}{\sqrt{T^{i+1}}}$, i.e., the sub-matrix is:
$$ \begin{bmatrix}
\frac{1}{\sqrt{T^2}} & \cdots & \frac{1}{\sqrt{T^2}} & 0 & \cdots & \cdots & \cdots & \cdots & 0 \\
\cdots & \cdots & \cdots & \cdots & \cdots & \cdots & \cdots & \cdots & \cdots \\
\frac{1}{\sqrt{T^2}} & \cdots & \frac{1}{\sqrt{T^2}} & 0 & \cdots & \cdots & \cdots & \cdots & 0 \\
0 & \cdots & 0 & \cdots & \cdots & \cdots & 0 & \cdots & 0 \\
\cdots & \cdots & \cdots & \cdots & \cdots & \cdots & \cdots & \cdots & \cdots \\
0 & \cdots & 0 & \cdots & \cdots & \cdots & 0 & \cdots & 0 \\
0 & \cdots & \cdots & \cdots & \cdots & 0 & \frac{1}{\sqrt{T^{s+1}}} & \cdots & \frac{1}{\sqrt{T^{s+1}}} \\
\cdots & \cdots & \cdots & \cdots & \cdots & \cdots & \cdots & \cdots & \cdots \\
0 & \cdots & \cdots & \cdots & \cdots & 0 & \frac{1}{\sqrt{T^{s+1}}} & \cdots & \frac{1}{\sqrt{T^{s+1}}} 
\end{bmatrix}.  $$

Then we consider the matrix $\bar{A}^\top \bar{A}$, it is still a block matrix. It has ${s}$ non-zero diagonal blocks and the size of each block is $T\times T$ The elements are $\frac{1}{T}$, i.e., the sub-matrix is 
$$ \begin{bmatrix}
\frac{1}{T} & \cdots & \frac{1}{T} & 0 & \cdots & \cdots & \cdots & \cdots & 0 \\
\cdots & \cdots & \cdots & \cdots & \cdots & \cdots & \cdots & \cdots & \cdots \\
\frac{1}{T} & \cdots & \frac{1}{T} & 0 & \cdots & \cdots & \cdots & \cdots & 0 \\
0 & \cdots & 0 & \cdots & \cdots & \cdots & 0 & \cdots & 0 \\
\cdots & \cdots & \cdots & \cdots & \cdots & \cdots & \cdots & \cdots & \cdots \\
0 & \cdots & 0 & \cdots & \cdots & \cdots & 0 & \cdots & 0 \\
0 & \cdots & \cdots & \cdots & \cdots & 0 & \frac{1}{T} & \cdots & \frac{1}{T} \\
\cdots & \cdots & \cdots & \cdots & \cdots & \cdots & \cdots & \cdots & \cdots \\
0 & \cdots & \cdots & \cdots & \cdots & 0 & \frac{1}{T} & \cdots & \frac{1}{T} 
\end{bmatrix}.  $$
With Lemma \ref{lem:block matrix}, the eigenvalues of the sub-matrix are 
    \begin{align*}
        &\sigma_1 = 1 ,\\
        &\sigma_2 = \cdots = \sigma_{s} = T \cdot \frac{1}{T} = 1,\\
        &\sigma_{s} = \cdots = \sigma_{sT\cdot sT} = 0.
    \end{align*}
Combing different the samples of different classes, the eigenvalues of $\bar{A}^\top \bar{A}$ are
\begin{align*}
        &\sigma_1 = \cdots =\sigma_{r} = 1 ,\\
        &\sigma_{r+1} = \cdots = \sigma_{r\cdot s} = T \cdot \frac{1}{T} = 1,\\
        &\sigma_{r\cdot s+1} = \cdots = \sigma_{sT\cdot sT} = 0.
\end{align*}
With the definition of singular values, the singular values of $\bar{A}$ are
$$
\begin{cases}
\sigma_j =1, &j \leq r s, \\
\sigma_j = 0,& \text{else}, \\
\end{cases}
$$
For masked models, the co-occurrence matrix size is $(\mathcal{C}_{s}^{(1-\rho_m)s}T^{(1-\rho_m)s} r \times ST  r)$. When the positions of conditional texts and target tokens are overlapped, the joint probability is 0. So each row in the co-occurrence matrix has $\mathcal{C}_{s-1}^{s_m}$ zeros, where $s_m = s(1-\rho_m)$. And the else elements is $\frac{1}{\mathcal{C}_{s-1}^{s_m}(T)^{s_m}sTr}$. By normalizing the co-occurrence matrix, the non-zero elements become $\bar{A'}_{X,X^+}= \frac{\sqrt{ST\mathcal{C}^{s_m}_s(T)^{s_m}}}{\mathcal{C}_{s-1}^{s_m}}$.

Similar to autoregressive SSL, we calculate the intra-class sub-matrix of $\bar{A'}^\top \bar{A'}$ and obtain:
\begin{small}
$$ \begin{pmatrix}
\frac{1}{(s-s_m)T} & \cdots & \frac{1}{(s-s_m)T} & \frac{(s-s_m-1)}{(s-s_m)(s-1)T} & \cdots & \frac{(s-s_m-1)}{(s-s_m)(s-1)T} & \cdots & \frac{(s-s_m-1)}{(s-s_m)(s-1)T} & \cdots & \frac{(s-s_m-1)}{(s-s_m)(s-1)T} \\
\cdots & \cdots & \cdots & \cdots & \cdots & \cdots & \cdots & \cdots & \cdots & \cdots \\
\frac{1}{(s-s_m)T} & \cdots & \frac{1}{(s-s_m)T} & \frac{(s-s_m-1)}{(s-s_m)(s-1)T} & \cdots & \frac{(s-s_m-1)}{(s-s_m)(s-1)T} & \cdots & \frac{(s-s_m-1)}{(s-s_m)(s-1)T} & \cdots & \frac{(s-s_m-1)}{(s-s_m)(s-1)T} \\
\frac{(s-s_m-1)}{(s-s_m)(s-1)T} & \cdots & \frac{(s-s_m-1)}{(s-s_m)(s-1)T} & \frac{1}{(s-s_m)T} & \cdots & \frac{1}{(s-s_m)T} & \cdots & \frac{(s-s_m-1)}{(s-s_m)(s-1)T} & \cdots & \frac{(s-s_m-1)}{(s-s_m)(s-1)T} \\
\cdots & \cdots & \cdots & \cdots & \cdots & \cdots & \cdots & \cdots & \cdots & \cdots \\
\frac{(s-s_m-1)}{(s-s_m)(s-1)T} & \cdots & \frac{(s-s_m-1)}{(s-s_m)(s-1)T} & \frac{1}{(s-s_m)T} & \cdots & \frac{1}{(s-s_m)T} & \cdots & \frac{(s-s_m-1)}{(s-s_m)(s-1)T} & \cdots & \frac{(s-s_m-1)}{(s-s_m)(s-1)T} \\
\cdots & \cdots & \cdots & \cdots & \cdots & \cdots & \cdots & \cdots & \cdots & \cdots \\
\frac{(s-s_m-1)}{(s-s_m)(s-1)T} & \cdots & \cdots & \cdots & \cdots & \frac{(s-s_m-1)}{(s-s_m)(s-1)T} & \cdots & \frac{1}{(s-s_m)T} & \cdots & \frac{1}{(s-s_m)T} \\
\cdots & \cdots & \cdots & \cdots & \cdots & \cdots & \cdots & \cdots & \cdots & \cdots \\
\frac{(s-s_m-1)}{(s-s_m)(s-1)T} & \cdots & \cdots & \cdots & \cdots & \frac{(s-s_m-1)}{(s-s_m)(s-1)T} & \cdots & \frac{1}{(s-s_m)T} & \cdots & \frac{1}{(s-s_m)T} 
\end{pmatrix} ,$$
\end{small}
With Lemma \ref{lem:block matrix}, the eigenvalues of the sub-matrix are 
    \begin{align*}
        &\sigma_1 = 1 ,\\
        &\sigma_2 = \cdots = \sigma_{s} = T \cdot \left(\frac{1}{(s-s_m)T}-\frac{s-1-s_m}{(s-s_m)(s-1)T}\right)= \frac{s_m}{(s-s_m)(s-1)}\,\\
        &\sigma_{s} = \cdots = \sigma_{sT\cdot sT} = 0.
    \end{align*}

Combing different the samples of different classes, the eigenvalues of $\bar{A'}^\top \bar{A'}$ are
\begin{align*}
        &\sigma_1 = \cdots =\sigma_{r} = 1 ,\\
        &\sigma_{r+1} = \cdots = \sigma_{r\cdot s} = \frac{s_m}{(s-s_m)(s-1)} ,\\
        &\sigma_{r\cdot s+1} = \cdots = \sigma_{sT\cdot sT} = 0.
\end{align*}

When $s_m<(s-1)$, we obtain $\frac{s_m}{(s-s_m)(s-1)}<1$. With the definition of singular values, the singular values of $\bar{A'}$ are
$$
\begin{cases}
\sigma_j =1,& j \leq r,\\
\displaystyle \sigma_j =\sqrt{\frac{s_m}{(s-s_m)(s-1)}}<1,& r < j \leq r s, \\
\sigma_j = 0,& \text{else}. \\
\end{cases}
$$

\end{proof}

\subsection{Proofs of Theorem \ref{thm:length mismatch}}
\begin{proof}
Given a sample $x_{i}$, let $X_{<k}^{i}$ be the first $k-1$ tokens of $x_{i}$, $(X_{k}^i)^{+}$ be the $k$-th token of $x_{i}$ and $(X_{k}^i)^-$ is any independent token. Expanding $-(Wf_{\text{mask}}(X_{<k}^{i}))^{\top }\mathbbm{1}_{(X_{k}^{i})^{+}}$ and we have
\begin{equation}
\begin{split}
-(Wf_{mask}(X_{<k}^{i}))^{\top }\mathbbm{1}_{(X_{k}^{i})^{+}}&=\frac{1}{2}\left\lVert Wf_{mask}(X_{<k}^{i})-\mathbbm{1}_{(X_{k}^{i})^{+}} \right\rVert _{2}^2 -\frac{1}{2} \left\lVert Wf_{mask}(X_{<k}^{i}) \right\rVert _{2}^2 - \frac{1}{2}\left\lVert \mathbbm{1}_{(X_{k}^{i})^{+}} \right\rVert _{2}^2\\
&=\frac{1}{2}\left\lVert Wf_{mask}(X_{<k}^{i})-\mathbbm{1}_{(X_{k}^{i})^{+}} \right\rVert _{2}^2-1.
\end{split}\label{equ:ineq1}
\end{equation}
Let the mask $m_{k}$ satisfy that $\{x_{i,1},\cdots ,x_{i,k-1}\}$ are all included in $x_{i}[m_{k}]$ and $x_{i,k}\in x_{i}[1-m_{k}]$. We denote $X_{m_{k}}^{i}=x_{i}[m_{k}]$ and have the following equation:
\begin{equation}
\frac{1}{2}\left\lVert Wf_{mask}(X_{<k}^{i})-\mathbbm{1}_{(X_{k}^{i})^{+}} \right\rVert _{2}^2\leq \frac{1}{2}\left\lVert Wf_{mask}(X_{m_{k}}^{i}) -Wf_{mask}(X_{<k}^{i})\right\rVert _{2}^2+\frac{1}{2}\left\lVert Wf_{mask}(X_{m_{k}}^{i})-\mathbbm{1}_{(X_{k}^{i})^{+}} \right\rVert _{2}^2,\label{equ:ineq3}
\end{equation}
The upper bound of $\displaystyle \frac{1}{2}\left\lVert Wf_{\text{mask}}(X_{m_{k}}^{i})-\mathbbm{1}_{(X_{k}^{i})^{+}} \right\rVert _{2}^2$ is:
\begin{equation}
\displaystyle \frac{1}{2}\left\lVert Wf_{mask}(X_{m_{k}}^{i})-\mathbbm{1}_{(X_{k}^{i})^{+}} \right\rVert _{2}^2=1-(Wf_{mask}(X_{m_{k}}^{i}))^{\top }\mathbbm{1}_{(X_{k}^{i})^{+}}\leq 1+\delta _{mask}.\label{equ:bound1}
\end{equation}
Next, we consider the term $\displaystyle \frac{1}{2}\left\lVert Wf_{mask}(X_{m_{k}}^{i}) -Wf_{mask}(X_{<k}^{i})\right\rVert _{2}^2=\frac{1}{2}\left\lVert W(f_{mask}(X_{m_{k}}^{i}) -f_{mask}(X_{<k}^{i}))\right\rVert _{2}^2$. We denote $X_{m,>k}^{i}=X_{m_{k}}^{i}-X_{<k}^{i}$ and $g(x_{i,u},x_{i,v},x_{i,w})=x_{i,u}W^{Q}(x_{i,v}W^{K})^{\top }x_{i,w}W^{V}$. Expanding $f_{\text{mask}}(X_{m_{k}}^{i})$ and $f_{\text{mask}}(X_{<k}^{i})$ separately and we obtain:
\begin{equation}
\begin{split}
f_{mask}(X_{m_{k}}^{i})&=X_{m_{k}}^{i}W^{Q}(X_{m_{k}}^{i}W^{K})^{\top }X_{m_{k}}^{i}W^{V}\\
&=\sum_{x_{i,u}\in X_{m_{k}}^{i}}^{}\sum_{x_{i,v}\in X_{m_{k}}^{i}}^{}\sum_{x_{i,w}\in X_{m_{k}}^{i}}^{}g(x_{i,u},x_{i,v},x_{i,w})\\
&=\sum_{x_{i,u}\in X_{<k}^{i}}^{}\sum_{x_{i,v}\in X_{m_{k}}^{i}}^{}\sum_{x_{i,w}\in X_{m_{k}}^{i}}^{}g(x_{i,u},x_{i,v},x_{i,w})+\sum_{x_{i,u}\in X_{m,>k}^{i}}^{}\sum_{x_{i,v}\in X_{m_{k}}^{i}}^{}\sum_{x_{i,w}\in X_{m_{k}}^{i}}^{}g(x_{i,u},x_{i,v},x_{i,w})\\
&=\sum_{x_{i,u}\in X_{<k}^{i}}^{}\sum_{x_{i,v}\in X_{<k}^{i}}^{}\sum_{x_{i,w}\in X_{<k}^{i}}^{}g(x_{i,u},x_{i,v},x_{i,w})+\sum_{x_{i,u}\in X_{<k}^{i}}^{}\sum_{x_{i,v}\in X_{m,>k}^{i}}^{}\sum_{x_{i,w}\in X_{<k}^{i}}^{}g(x_{i,u},x_{i,v},x_{i,w})\\
&+\sum_{x_{i,u}\in X_{<k}^{i}}^{}\sum_{x_{i,v}\in X_{<k}^{i}}^{}\sum_{x_{i,w}\in X_{m,>k}^{i}}^{}g(x_{i,u},x_{i,v},x_{i,w})+\sum_{x_{i,u}\in X_{<k}^{i}}^{}\sum_{x_{i,v}\in X_{m,>k}^{i}}^{}\sum_{x_{i,w}\in X_{m,>k}^{i}}^{}g(x_{i,u},x_{i,v},x_{i,w})\\
&+\sum_{x_{i,u}\in X_{m,>k}^{i}}^{}\sum_{x_{i,v}\in X_{m_{k}}^{i}}^{}\sum_{x_{i,w}\in X_{m_{k}}^{i}}^{}g(x_{i,u},x_{i,v},x_{i,w})
\end{split}\label{equ:mkexpand}
\end{equation}
with
\begin{equation}
f_{mask}(X_{<k}^{i})=\sum_{x_{i,u}\in X_{<k}^{i}}^{}\sum_{x_{i,v}\in X_{<k}^{i}}^{}\sum_{x_{i,w}\in X_{<k}^{i}}^{}g(x_{i,u},x_{i,v},x_{i,w}).
\end{equation}
For any $(x_{i,u},x_{i,v},x_{i,w})$ and $(x_{i,a},x_{i,b},x_{i,c})$, we define:
\begin{equation*}
\varepsilon _{a,b,c}^{u,v,w}=g(x_{i,u},x_{i,v},x_{i,w})-g(x_{i,a},x_{i,b},x_{i,c}).
\end{equation*}
We can obtain:
\begin{equation*}
\begin{split}
\sum_{x_{i,u}\in X_{<k}^{i}}^{}\sum_{x_{i,v}\in X_{m,>k}^{i}}^{}\sum_{x_{i,w}\in X_{<k}^{i}}^{}g(x_{i,u},x_{i,v},x_{i,w})&=\left\lvert X_{<k}^{i} \right\rvert \left\lvert X_{m,>k}^{i} \right\rvert \left\lvert X_{<k}^{i} \right\rvert g(x_{i,a},x_{i,b},x_{i,c})\\
&+\sum_{x_{i,u}\in X_{<k}^{i}}^{}\sum_{x_{i,v}\in X_{m,>k}^{i}}^{}\sum_{x_{i,w}\in X_{<k}^{i}}^{}\varepsilon ^{u,v,w}_{a,b,c},
\end{split}
\end{equation*}
\begin{equation*}
\begin{split}
\sum_{x_{i,u}\in X_{<k}^{i}}^{}\sum_{x_{i,v}\in X_{<k}^{i}}^{}\sum_{x_{i,w}\in X_{m,>k}^{i}}^{}g(x_{i,u},x_{i,v},x_{i,w})&=\left\lvert X_{<k}^{i} \right\rvert \left\lvert X_{<k}^{i} \right\rvert \left\lvert X_{m,>k}^{i} \right\rvert g(x_{i,a},x_{i,b},x_{i,c})\\
&+\sum_{x_{i,u}\in X_{<k}^{i}}^{}\sum_{x_{i,v}\in X_{<k}^{i}}^{}\sum_{x_{i,w}\in X_{m,>k}^{i}}^{}\varepsilon ^{u,v,w}_{a,b,c},
\end{split}
\end{equation*}
\begin{equation*}
\begin{split}
\sum_{x_{i,u}\in X_{<k}^{i}}^{}\sum_{x_{i,v}\in X_{m,>k}^{i}}^{}\sum_{x_{i,w}\in X_{m,>k}^{i}}^{}g(x_{i,u},x_{i,v},x_{i,w})&=\left\lvert X_{<k}^{i} \right\rvert \left\lvert X_{m,>k}^{i} \right\rvert \left\lvert X_{m,>k}^{i} \right\rvert g(x_{i,a},x_{i,b},x_{i,c})\\
&+\sum_{x_{i,u}\in X_{<k}^{i}}^{}\sum_{x_{i,v}\in X_{m,>k}^{i}}^{}\sum_{x_{i,w}\in X_{m,>k}^{i}}^{}\varepsilon ^{u,v,w}_{a,b,c},
\end{split}
\end{equation*}
\begin{equation*}
\begin{split}
\sum_{x_{i,u}\in X_{m,>k}^{i}}^{}\sum_{x_{i,v}\in X_{m_{k}}^{i}}^{}\sum_{x_{i,w}\in X_{m_{k}}^{i}}^{}g(x_{i,u},x_{i,v},x_{i,w})&=\left\lvert X_{m,>k}^{i} \right\rvert \left\lvert X_{m_{k}}^{i} \right\rvert \left\lvert X_{m_{k}}^{i} \right\rvert g(x_{i,a},x_{i,b},x_{i,c})\\
&+\sum_{x_{i,u}\in X_{m,>k}^{i}}^{}\sum_{x_{i,v}\in X_{m_{k}}^{i}}^{}\sum_{x_{i,w}\in X_{m_{k}}^{i}}^{}\varepsilon ^{u,v,w}_{a,b,c},
\end{split}
\end{equation*}
where $\{x_{i,a},x_{i,b},x_{i,c}\}\in X_{<k}^{i}$,$\left\lvert X_{<k}^{i} \right\rvert =k-1,\left\lvert X_{m_{k}}^{i} \right\rvert =s(1-\rho _{m}),\left\lvert X_{m,>k}^{i}\right\rvert =s(1-\rho _{m})-k+1$. It is easy to to calculate the summation:
\begin{equation*}
\begin{split}
&\left\lvert X_{<k}^{i} \right\rvert \left\lvert X_{m,>k}^{i} \right\rvert \left\lvert X_{<k}^{i} \right\rvert+\left\lvert X_{<k}^{i} \right\rvert \left\lvert X_{<k}^{i} \right\rvert \left\lvert X_{m,>k}^{i} \right\rvert+\left\lvert X_{<k}^{i} \right\rvert \left\lvert X_{m,>k}^{i} \right\rvert \left\lvert X_{m,>k}^{i} \right\rvert+\left\lvert X_{m,>k}^{i} \right\rvert \left\lvert X_{m_{k}}^{i} \right\rvert \left\lvert X_{m_{k}}^{i} \right\rvert\\
&=(s(1-\rho _{m}))^{3}-(k-1)^{3}.
\end{split}
\end{equation*}
Let
\begin{equation*}
\begin{split}
\varepsilon _{a,b,c}&=\sum_{x_{i,u}\in X_{<k}^{i}}^{}\sum_{x_{i,v}\in X_{m,>k}^{i}}^{}\sum_{x_{i,w}\in X_{<k}^{i}}^{}\varepsilon ^{u,v,w}_{a,b,c}+\sum_{x_{i,u}\in X_{<k}^{i}}^{}\sum_{x_{i,v}\in X_{<k}^{i}}^{}\sum_{x_{i,w}\in X_{m,>k}^{i}}^{}\varepsilon ^{u,v,w}_{a,b,c}\\
&+\sum_{x_{i,u}\in X_{<k}^{i}}^{}\sum_{x_{i,v}\in X_{m,>k}^{i}}^{}\sum_{x_{i,w}\in X_{m,>k}^{i}}^{}\varepsilon ^{u,v,w}_{a,b,c}+\sum_{x_{i,u}\in X_{m,>k}^{i}}^{}\sum_{x_{i,v}\in X_{m_{k}}^{i}}^{}\sum_{x_{i,w}\in X_{m_{k}}^{i}}^{}\varepsilon ^{u,v,w}_{a,b,c}
\end{split}
\end{equation*}
we have
\begin{equation}
\begin{split}
&\sum_{x_{i,u}\in X_{<k}^{i}}^{}\sum_{x_{i,v}\in X_{m,>k}^{i}}^{}\sum_{x_{i,w}\in X_{<k}^{i}}^{}g(x_{i,u},x_{i,v},x_{i,w})+\sum_{x_{i,u}\in X_{<k}^{i}}^{}\sum_{x_{i,v}\in X_{<k}^{i}}^{}\sum_{x_{i,w}\in X_{m,>k}^{i}}^{}g(x_{i,u},x_{i,v},x_{i,w})\\+&\sum_{x_{i,u}\in X_{<k}^{i}}^{}\sum_{x_{i,v}\in X_{m,>k}^{i}}^{}\sum_{x_{i,w}\in X_{m,>k}^{i}}^{}g(x_{i,u},x_{i,v},x_{i,w})+\sum_{x_{i,u}\in X_{m,>k}^{i}}^{}\sum_{x_{i,v}\in X_{m_{k}}^{i}}^{}\sum_{x_{i,w}\in X_{m_{k}}^{i}}^{}g(x_{i,u},x_{i,v},x_{i,w})\\
=&((s(1-\rho _{m}))^{3}-(k-1)^{3})g(x_{i,a},x_{i,b},x_{i,c})+\varepsilon _{a,b,c},
\end{split}\label{equ:expand1}
\end{equation}
with
\begin{equation}
\left\lVert \varepsilon _{a,b,c} \right\rVert _{2}^{2}\leq ((s(1-\rho _{m}))^{3}-(k-1)^{3})\eta .\label{equ:bound2}
\end{equation}
Since \eqref{equ:expand1} holds for any $x_{i,a},x_{i,b},x_{i,c}\in X_{<k}^{i}$, taking the average over $(x_{i,a},x_{i,b},x_{i,c})$ yields
\begin{equation}
\begin{split}
&\sum_{x_{i,u}\in X_{<k}^{i}}^{}\sum_{x_{i,v}\in X_{m,>k}^{i}}^{}\sum_{x_{i,w}\in X_{<k}^{i}}^{}g(x_{i,u},x_{i,v},x_{i,w})+\sum_{x_{i,u}\in X_{<k}^{i}}^{}\sum_{x_{i,v}\in X_{<k}^{i}}^{}\sum_{x_{i,w}\in X_{m,>k}^{i}}^{}g(x_{i,u},x_{i,v},x_{i,w})\\+&\sum_{x_{i,u}\in X_{<k}^{i}}^{}\sum_{x_{i,v}\in X_{m,>k}^{i}}^{}\sum_{x_{i,w}\in X_{m,>k}^{i}}^{}g(x_{i,u},x_{i,v},x_{i,w})+\sum_{x_{i,u}\in X_{m,>k}^{i}}^{}\sum_{x_{i,v}\in X_{m_{k}}^{i}}^{}\sum_{x_{i,w}\in X_{m_{k}}^{i}}^{}g(x_{i,u},x_{i,v},x_{i,w})\\
=&\left(\frac{(s(1-\rho _{m}))^{3}}{(k-1)^{3}}-1\right)f_{mask}(X_{<k}^{i})+\frac{1}{(k-1)^3}\sum_{x_{i,a}\in X_{<k}^{i}}^{}\sum_{x_{i,b}\in X_{<k}^{i}}^{}\sum_{x_{i,c}\in X_{<k}^{i}}^{}\varepsilon _{a,b,c}.
\end{split}\label{equ:expand2}
\end{equation}
Let $\varepsilon_{k} =\frac{1}{(k-1)^3}\sum_{x_{i,a}\in X_{<k}^{i}}^{}\sum_{x_{i,b}\in X_{<k}^{i}}^{}\sum_{x_{i,c}\in X_{<k}^{i}}^{}\varepsilon _{a,b,c}$. Using \eqref{equ:bound2} we have
\begin{equation}
\left\lVert \varepsilon _{k} \right\rVert _{2}^2\leq ((s(1-\rho _{m}))^{3}-(k-1)^{3})\eta.
\end{equation}
Combining \eqref{equ:expand2} and \eqref{equ:mkexpand} we have
\begin{equation*}
f_{mask}(X_{m_{k}}^{i})=\frac{(s(1-\rho _{m}))^{3}}{(k-1)^{3}}f_{mask}(X_{<k}^{i})+\varepsilon _{k},
\end{equation*}
which implies that
\begin{equation}
f_{mask}(X_{m_{k}}^{i}) -f_{mask}(X_{<k}^{i})=\left(\frac{(s(1-\rho _{m}))^{3}}{(k-1)^{3}}-1\right)f_{mask}(X_{m_{k}}^{i})-\varepsilon _{k}.
\end{equation}
The upper bound of $\displaystyle \frac{1}{2}\left\lVert Wf_{mask}(X_{m_{k}}^{i}) -Wf_{mask}(X_{<k}^{i})\right\rVert _{2}^2$ is given by
\begin{equation*}
\begin{split}
\frac{1}{2}&\left\lVert Wf_{mask}(X_{m_{k}}^{i}) -Wf_{mask}(X_{<k}^{i})\right\rVert _{2}^2=\frac{1}{2}\left\lVert \left(\frac{(s(1-\rho _{m}))^{3}}{(k-1)^{3}}-1\right)Wf_{mask}(X_{m_{k}}^{i})-W\varepsilon _{k}\right\rVert _{2}^2\\
&\leq \frac{1}{2}\left\lVert \left(\frac{(s(1-\rho _{m}))^{3}}{(k-1)^{3}}-1\right)Wf_{mask}(X_{m_{k}}^{i}) \right\rVert _{2}^2 +\frac{1}{2}\left\lVert W\varepsilon _{k} \right\rVert _{2}^2\\
&\leq \frac{1}{2}\left(\frac{(s(1-\rho _{m}))^{3}-(k-1)^3}{(k-1)^{3}}\right)^{2}+\frac{1}{2}((s(1-\rho _{m}))^{3}-(k-1)^{3})\left\lVert W \right\rVert _{2}^2\eta\\
&=\frac{w_{k}^{2}}{2(k-1)^{6}}+\frac{w_{k}}{2}\left\lVert W \right\rVert _{2}^2 \eta.
\end{split}
\end{equation*}
This result, combined with \eqref{equ:ineq1} to \eqref{equ:bound1}, imply that
\begin{equation}
-(Wf_{\text{mask}}(X_{<k}^{i}))^{\top }\mathbbm{1}_{(X_{k}^{i})^{+}}\leq \frac{w_{k}^{2}}{2(k-1)^{6}}+\frac{w_{k}}{2}\left\lVert W \right\rVert _{2}^2 \eta+\delta_{mask}
\end{equation}
Meanwhile, the upper bound of $((Wf_{mask}(X_{<k}^{i}))^{\top }\mathbbm{1}_{(X_{k}^{i})^{-}})^{2}$ is given by:
\begin{equation}
((Wf_{mask}(X_{<k}^{i}))^{\top }\mathbbm{1}_{(X_{k}^{i})^{-}})^{2}\leq \left\lVert Wf_{mask}(X_{<k}^{i}) \right\rVert _{2}^2\left\lVert \mathbbm{1}_{(X_{k}^{i})^{-}} \right\rVert _{2}^2=1.
\label{equ:ineq2}
\end{equation}
The two upper bounds imply the following result:
\begin{equation}
\begin{split}
\mathcal{L}_{gen}(f_{mask})&=-\mathbb{E}_{(X_{<k}^{i},(X_{k}^{i})^{+})}(Wf_{\text{mask}}(X_{<k}^{i}))^{\top }\mathbbm{1}_{(X_{k}^{i})^{+}}+\mathbb{E}_{X_{<k}^{i},(X_{k}^{i})^{-}}((Wf_{\text{mask}}(X_{<k}^{i}))^{\top }\mathbbm{1}_{(X_{k}^{i})^{-}})^{2}\\
&=-\mathbb{E}_{x_i} \frac{\sum_k (Wf_{\text{mask}}(X_{<k}^{i}))^{\top }\mathbbm{1}_{(X_{k}^{i})^{+}}}{s(1-\rho_m)}+\mathbb{E}_{X_{<k}^{i},(X_{k}^{i})^{-}}((Wf_{\text{mask}}(X_{<k}^{i}))^{\top }\mathbbm{1}_{(X_{k}^{i})^{-}})^{2}\\
&\leq \frac{\sum_{k}^{}\left(\frac{w_{k}^{2}}{(k-1)^{6}}+w_{k}\left\lVert W \right\rVert _{2}^2 \eta\right)}{2s(1-\rho _{m})}+\delta _{mask}+1.
\end{split}
\end{equation}
which finishes the proof.
\end{proof}

\section{Experimental Details}

\subsection{Model details in Table \ref{tab:cls_and_gen}}
\label{appendix-comparison-detail}
In the comparison of GLUE score, the autoregressive model is chosen to be GPT-2 medium (345M parameters) \citep{radford2019language} and the masked model is chosen to be BERT large (340M parameters) \citep{bert}. In the comparison of perplexity, the autoregressive model is chosen to be GPT-2 medium (345M parameters) and the masked model is chosen to be BERT-Large-CAS (395M parameters) \citep{lmtransformer2019}.

\subsection{More details in Figure \ref{estimated connectivity}}
\label{appendix-estimated-detail}
We employ GPT and BERT models that have been trained on a subset of the Pile dataset \citep{gao2020pile} using the Cramming protocol \citep{geiping2023cramming}. Following pretraining, we randomly sample several examples from the remaining Pile dataset and obtain output features using the pretrained models. We choose 8000 features for each model and calculate their similarity by computing the inner product for every pair of features. Subsequently, we sort these similarities and select the top $10^5$ pairs to form Figure \ref{estimated connectivity}.

\subsection{More details on Two-Stream Attention Transformer.}
\label{appendix-two-stream}    
    In our objective of Equation \ref{equation-semi-autoregressive}, we predict several tokens given one conditional sequence. It seems that we can modify the causal mask to suit the dependency relationship between tokens. But in fact, we are not capable of modeling the objective if we use a single stream. This is because of two requirements that are contradictory in a standard decoder-only Transformer architecture, which have been discussed in XLNet \citep{yang2020xlnet}: (1) Suppose the output of the network in position $z_t$ is parameterized by $g_{\theta}(x_{\mathbf{z}<t}, z_t)$. In order to predict the token $x_{z_t}$, the network output in the $z_t$ position in one layer $g_{\theta}(x_{\mathbf{z}<t}, z_t)$ should only use the position $z_t$ and not the content $x_{z_t}$, (2) to predict the other tokens $x_{z_j}$ with $j > t$, $g_{\theta}(x_{\mathbf{z}<t}, z_t)$ should also encode the content $x_{z_t}$ to provide full contextual information. Therefore, it is not enough with only one stream. Since the two-stream attention proposed by XLNet facing the same problem works well, we build upon the two-stream attention mechanism proposed by XLNet and extend it by excavating important usage of the causal masks.

    In the two-stream attention Transformer, the content stream $h_\theta$ encodes the contextual information, and the query stream $g_\theta$ predicts the targets with the help of the content stream. The two stream share the same weights of the attention block but differ in the causal mask. Intuitively, the causal masks in the two streams enable us to easily formulate various dependency relationship between tokens. The main difference for the two causal masks is that the content stream should ensure $x_{z_t}$ can be attended to itself while the query stream is the contrary. The two-stream attention allows us to arbitrarily change the sequential order without confronting the inconsistency as mentioned above. For each self-attention layer $l=1,\dots,L$, the two streams of representations are schematically updated with a shared set of parameters. In the $l$-th layer, the outputs of a self-attention head $\bA_h^{(l)}$ and $\bA_g^{(l)}$ in the two-stream are computed in the form of:
    \begin{equation}
    \begin{aligned}
        \bQ_h&=\bH^{(l-1)}\bW^l_Q,\  \bK_h=\bH^{(l-1)}\bW^l_K,\  \bV_h=\bH^{(l-1)}\bW^l_V\\
        \bQ_g&=\bG^{(l-1)}\bW^l_Q,\  \bK_g=\bH^{(l-1)}\bW^l_K,\  \bV_g=\bH^{(l-1)}\bW^l_V\\
        \bA_h^{(l)}&=\operatorname{softmax}(\frac{\bQ_h\bK_h^{\top}}{\sqrt{d_k}}+\bM_h)\bV_h,
        \bA_g^{(l)}=\operatorname{softmax}(\frac{\bQ_g\bK_g^{\top}}{\sqrt{d_k}}+\bM_g)\bV_g,
    \end{aligned}
    \end{equation}
    where $\bH^{(l-1)}$ and $\bG^{(l-1)}$, the representations before the $(l-1)$-th layer of the content stream and the query stream, are linearly projected to queries, keys and values using shared trainable parameter matrices $\bW^l_Q,\bW^l_K,\bW^l_V$ respectively, and $d_k$ is the dimension of the representations. $\bM_h, \bM_g\in \mathbb{R}^{s\times s}$ are the causal masks in the content stream and the query stream, respectively, which are the key factors in our framework. The value of each element $(\bM)_{ij}$ in either causal mask can only be $0$ or $-\infty$, where $(\bM)_{ij}=0$ means the $j-$th token \textbf{can} be attended to the $i-$th token and $(\bM)_{ij}=-\infty$ means the $j-$th token \textbf{cannot} be attended to the $i-$th token. This indicates that the causal mask has a one-to-one correspondence with the dependency relationship between tokens. Therefore, \textbf{we only need to alter the causal masks $\bM_h$ and $\bM_g$ to adapt to different choices of the group setting, which eliminate the need of using multiple architectures.} We will show how to construct the corresponding attention masks.

    Note that the value of each element $(\bM)_{ij}$ in either causal mask can only be $0$ or $-\infty$, where $(\bM)_{ij}=0$ means the $j-$th token \textbf{can} be attended to the $i-$th token and $(\bM)_{ij}=-\infty$ means the $j-$th token \textbf{cannot} be attended to the $i-$th token. This implicates that by finding out the dependency between each token we can determine the causal mask. We provide the causal masks corresponding to Equation (\ref{equation-semi-autoregressive}): Let $f(t)$ denote the index of the group to which $x_t$ belongs. In this situation, $x_i$ will be dependent on $x_j$ if $f(i)>f(j)$. Besides, tokens in the same group is dependent on each other in the content stream to provide contextual information. The causal masks for the two streams should be
    \begin{equation}
    \label{attention-mask-group}
        (\bM_h)_{ij}=\left\{\begin{array}{ll}
            0 & f(i)\geq f(j)>0 \\
            -\infty & f(j)>f(i)>0
        \end{array}\right.,\ 
        (\bM_g)_{ij}=\left\{\begin{array}{ll}
            0 & f(i)> f(j)>0 \\
            -\infty & f(j)\geq f(i)>0
        \end{array}\right.
    \end{equation}

\section{Ablation Study}
We conduct more ablation experiments on the choice of $t$ in the vision domain. We vary $t$ in $[1,2,3,4]$ and investigate the mixture of $t=1$ and $t=2$. We use ImageNet as the dataset. The results are shown in the Table \ref{tab:ablation-t}. The results indicate that there exists a sweet point between $t=2$ and $t=3$. This means that $t$ should not be too large. The mixture choice of $t$ does not bring much improvement as well.

 \begin{equation}
     \Vert \tilde P_M - F_VF_L^\top \Vert ^2
\end{equation}

\begin{table}[H]
    \centering
    \caption{Ablation study on the choice of $t$ in the diversity-enhanced autoregressive objective.}
    \begin{tabular}{lcc}
    \toprule
        $t$ & Linear probing accuracy & Finetuning accuracy \\
    \midrule
        1 (original autoregressive objective) & 56.2 & 82.5 \\
        2 (used in the previous experiments) & 59.4 & 82.9 \\
        3 & 59.6 & 82.7 \\
        4 & 58.8 & 82.5 \\
        mixture of 1 and 2 & 58.1 & 82.5 \\
    \bottomrule
    \end{tabular}
    \label{tab:ablation-t}
\end{table}

\end{document}